\documentclass{article}

\usepackage{xcolor}

\usepackage{microtype}
\usepackage{graphicx}
\usepackage{booktabs} 
\usepackage{algorithm}
\usepackage{algorithmicx}
\usepackage{algpseudocode}
\usepackage{amssymb}
\usepackage{wrapfig}
\usepackage{diagbox}
\usepackage{amsthm}
\usepackage{thm-restate}

\usepackage{hyperref}

\usepackage[accepted]{icml2021}

\usepackage{caption}
\usepackage{subcaption}
\usepackage{amsmath}
\usepackage{mathabx} 
\usepackage{multirow}
\usepackage[capitalise]{cleveref}
\crefname{equation}{Eq.\!}{Eqs.\!}
\crefname{figure}{Fig.\!}{Figs.\!}
\usepackage{enumitem} 
\setlist{leftmargin=*,nolistsep,noitemsep,topsep=0em}  
\usepackage{natbib} 
\usepackage[euler-digits,euler-hat-accent]{eulervm} 

\usepackage{definitions}

\newif\ifmaintextonly
\maintextonlyfalse


\newcommand{\longname}{Residual Actor-Critic}
\newcommand{\shortname}{Res-AC}

\icmltitlerunning{Characterizing the Gap Between Actor-Critic and Policy Gradient}

\begin{document}

\twocolumn[
\icmltitle{Characterizing the Gap Between Actor-Critic and Policy Gradient}



\icmlsetsymbol{equal}{*}

\begin{icmlauthorlist}
\icmlauthor{Junfeng Wen}{ua}
\icmlauthor{Saurabh Kumar}{stf}
\icmlauthor{Ramki Gummadi}{goo}
\icmlauthor{Dale Schuurmans}{ua,goo}
\end{icmlauthorlist}

\icmlaffiliation{ua}{Department of Computing Science, University of Alberta, Edmonton, Canada}
\icmlaffiliation{stf}{Stanford University}
\icmlaffiliation{goo}{Google Brain}

\icmlcorrespondingauthor{Junfeng Wen}{junfengwen@gmail.com}

\icmlkeywords{Actor-Critic}

\vskip 0.3in
]



\printAffiliationsAndNotice{}  

\begin{abstract}
Actor-critic (AC) methods are ubiquitous in reinforcement learning.
Although it is understood that AC methods are closely related to policy gradient (PG), their precise connection has not been fully characterized previously.
In this paper, we explain the gap between AC and PG methods by identifying the exact adjustment to the AC objective/gradient that recovers the true policy gradient of the cumulative reward objective (PG).
Furthermore, by viewing the AC method as a two-player Stackelberg game between the actor and critic, we show that the Stackelberg policy gradient can be recovered as a special case of our more general analysis.
Based on these results, we develop practical algorithms, \emph{\longname} and \emph{Stackelberg Actor-Critic}, for estimating the correction between AC and PG and use these to modify the standard AC algorithm.
Experiments on popular tabular and continuous environments show the proposed corrections can improve both the sample efficiency and final performance of existing AC methods.

\end{abstract}

\section{Introduction}

Policy gradient~(PG)~\citep{marbach2001simulation} is the foundation of many reinforcement learning~(RL) algorithms~\citep{sutton2018reinforcement}. 
The basic
PG method~\citep{sutton2000policy} requires access to the state-action values of the current policy, which is challenging to obtain. 
For instance, Monte-Carlo value estimates~\citep{Williams92simplestatistical} are unbiased but suffer from high variance and low sample efficiency. 
To address this issue, actor-critic~(AC) methods learn a parametrized value function~(critic) to estimate the state-action values. 
This approach has inspired many successful  algorithms~\citep{mnih2016asynchronous,lillicrap2016continuous,haarnoja2018soft} that achieve impressive performance on a range of challenging tasks. Despite the success of AC methods, AC and PG have subtle differences that are only partially characterized in the literature \citep{konda2000actor, sutton2000policy}. 
In particular, the distinction between PG and practical AC methods which use an arbitrarily parametrized critic (e.g.\ high-capacity function approximators such as neural networks) is unclear. 
In this paper, we investigate
the intuition that explicitly quantifying this difference 
and minimizing it to increase AC's fidelity to PG can benefit practical AC methods. 

A key difficulty in understanding
the difference between AC and PG is the use of non-linear function approximation.
A non-linear
parametrization violates the compatibility requirement between the actor and critic \citep{sutton2000policy} needed to ensure 
equivalence of PG and the policy improvement step in AC. 
Additionally, the critic in AC, which estimates the policy's state-action values, can be highly inaccurate since it may suffer from bias and is not optimized to convergence. 
Consequently,
the policy improvement step of AC may differ substantially from the corresponding PG update which uses the true state-action values.   

In this paper, we precisely characterize the gap between these two policy improvement updates. 
We start by investigating
natural objective functions 
for AC and show that several classic algorithms can 
just be seen as alternative schedules for
the actor and critic updates under these objectives (\cref{tab:algo_relations}). 
From these observations, we then 
quantify the gap between AC and PG from both objective and gradient perspectives. 
From the objective perspective, we calculate 
the difference between
the actor objective 
and
the original cumulative reward objective used to derive PG; 
while from the gradient perspective, we identify the difference of the policy improvement update in AC when using an arbitrary critic versus following the true policy gradient.

Given this
understanding of the difference between PG and AC, we propose two solutions to close their gap and reduce the bias introduced by the critic in practice. 
First, we develop novel update rules for AC that estimate this gap and add a correction
to standard updates used in AC;
in particular,
we propose a new AC framework, which we call Residual Actor-Critic (Res-AC). 
Second,
we explore AC from a game-theoretic perspective and treat AC as a Stackelberg game~\citep{fiez2020implicit,sinha2017review}. 
By treating the actor and critic as two players, 
we propose a second novel AC framework, Stackelberg Actor-Critic (Stack-AC), and prove that the Stack-AC updates can also close the gap between AC and PG under certain assumptions. 
We implement the Res-AC and Stack-AC update rules by applying them to Soft Actor-Critic~\citep{haarnoja2018soft}. 
We present empirical results that show these modifications 
can improve sample efficiency and final performance in both a tabular domain as well as continuous control tasks which require neural networks to approximate the actor and critic.

\section{Background}\label{sec:background}

Throughout this paper we exploit a matrix-vector notation that 
significantly simplifies the calculations and clarifies the exposition.
Therefore, we first need to introduce the notation and relevant
concepts we leverage in some detail.
Beyond matrix-vector notations, 
\ifmaintextonly
Appendix A
\else
\cref{app:notation} 
\fi
provides alternative representations for some of the key concepts in this paper.

\textbf{Markov Decision Process}. 
Let $\RR_+=\{x\in\RR|x\ge 0\}$.
A \emph{Markov Decision Process} (MDP) is defined as $\Mcal=\{\Scal,\Acal,\vr,P,\vmu_0,\gamma\}$, 
where $\Scal$ is the state space, 
$\Acal$ is the action space, 
$\vr=[\vr_{s_1}^\top,\dots,\vr_{s_{|\Scal|}}^\top]^\top\in\RR^{|\Scal||\Acal|\times 1}$ is reward vector/function,\footnote{
A vector can represent a concept in tabular (finite dimension) or continuous space (infinite dimension). Correspondingly, the inner product $\vx^\top\vy$ can be interpreted as summation or integral.}
$P\in\RR_+^{|\Scal||\Acal|\times|\Scal|}$ is the transition matrix where $P_{sa,\widetilde{s}}=\Pr(\widetilde{s}|s,a)$ for some state $s\in\Scal$, action $a\in\Acal$ and next state $\widetilde{s}\in\Scal$,
$\vmu_0\in\RR_+^{|\Scal|\times 1}$ is the initial state distribution,
and $\gamma\in[0,1)$ is the discount factor.
It is clear that $\vmu_0^\top\onevec_{|\Scal|}=1$ and $P\onevec_{|\Scal|}=\onevec_{|\Scal||\Acal|}$ since $\vmu_0$ and $P$ represent probabilities, where $\onevec$ is the vector of all ones and the subscript represents its dimensionality.

\textbf{Actor and Critic}. 
A \emph{policy} or \emph{actor} can be represented as 
$\vpi\in\RR_+^{|\Scal||\Acal|\times 1}$, 
and
$\vpi_{s_1}\in\RR_+^{|\Acal|\times 1}$ denotes the policy for the first state $s_1$. 
We denote the expanded matrix of $\vpi$ as
\begin{equation}
\setlength\arraycolsep{1pt}
\Pi=
\begin{bmatrix}
\vpi^\top_{s_1} & & \zerovec \\[-1pt]
& \ddots & \\[-1pt]
\zerovec & & \vpi^\top_{s_{|\Scal|}} \\
\end{bmatrix}
\in\RR_+^{|\Scal|\times|\Scal||\Acal|}
\label{eq:pi_expanded}
\end{equation}
We also use
$\vpi_\theta$
to denote
a parametrized policy with parameters $\vtheta$.
In the case of a tabular softmax policy, $\vpi_\theta$ are the per-state softmax transformation of the logits $\vtheta\in\RR^{|\Scal||\Acal|\times 1}$. 
We assume that $\vpi_\theta$ is properly parametrized 
so that it is normalized for every state.\footnote{For example, using softmax transformation for discrete action or Beta distribution for box-constrained action~\citep{chou2017improving}.}
A policy's state-action values (a.k.a.\ $Q$-values) are denoted as $\vq_\theta=\sum_{i=0}^{\infty}(\gamma P\Pi_\theta)^i\vr \in \RR^{|\Scal||\Acal| \times 1}$.
We construct an block-diagonal matrix $Q_\theta\in\RR^{|\Scal|\times|\Scal||\Acal|}$ similar to \cref{eq:pi_expanded} for $\vq_\theta$.
In general $\vq_\theta$ is not available, so one may approximate it using a \emph{critic} $\vq_\phi$ with parameters $\vphi$. 
In the tabular case, $\vq_\phi=\vphi\in\RR^{|\Scal||\Acal|\times 1}$ is a lookup table.
The actor and critic Jacobian matrices are respectively
\begin{align}
(H_\theta)_{i,sa}
=\left[ \frac{\partial (\vpi_\theta)_{sa}}{\partial\theta_i} \right]
\quad
(H_\phi)_{i,sa}
=\left[ \frac{\partial (\vq_\phi)_{sa}}{\partial\phi_i} \right]
\label{eq:actor_jacobian}
\end{align}

An important concept that will be used repeatedly throughout the paper is the (discounted) stationary distribution of a policy $\vpi_\theta$ over all states and actions in the MDP, where initial states are sampled from $\vmu_0$. This distribution is denoted as $\vd_\theta\in\RR_+^{|\Scal||\Acal|\times 1}$ and it is defined as follows:
\begin{align}
d_\theta(s,a)
\!=\!
(1\!-\!\gamma)\!\sum_{i=0}^{\infty}\gamma^i
\Pr(S_i\!=\!s,A_i\!=\!a|P,\vpi_\theta,S_0\!\sim\!\vmu_0).
\nonumber
\end{align} 
A policy's stationary distribution satisfies the following recursion~\citep{wang2007dual}:
\begin{align}
\vd_\theta 
= (1-\gamma)\Pi_\theta^\top\vmu_0
+\gamma \Pi_\theta^\top P^\top \vd_\theta
\label{eq:d_recursion_matrix}
\end{align}
We use $\vd_{\Scal,\theta}\in\RR_+^{|\Scal|\times 1}$ to denote the stationary distribution over the states instead of the state-action pairs.
More precisely, $\vd_{\Scal,\theta} = \Xi \vd_\theta$, where $\Xi$ is the marginalization matrix
\begin{equation}
\setlength\arraycolsep{1pt}
\Xi=
\begin{bmatrix}
\onevec^\top_{|\Acal|} & & \zerovec\\[-1pt]
& \ddots & \\[-1pt]
\zerovec & & \onevec^\top_{|\Acal|} \\
\end{bmatrix}
\in\RR^{|\Scal|\times|\Scal||\Acal|}
\end{equation}
Furthermore, we use
$D_\theta=\Delta(\vd_\theta)$ where $\Delta(\cdot)$
maps a vector to a diagonal matrix with its elements on the main diagonal.

\textbf{Policy Objective and Policy Gradient}. 
The cumulative (discounted) reward objective, or \textit{policy objective}, of a policy $\vpi_\theta$ can be written as~\citep{puterman2014markov}
\begin{align}
\max_{\vtheta}\ J(\vtheta)
&\defeq(1-\gamma)\vmu_0^\top\Pi_\theta\sum_{i=0}^\infty(\gamma P\Pi_\theta)^i\vr\\
&=(1-\gamma)\vmu_0^\top\Pi_\theta\vq_\theta
\label{eq:cumulative_obj}
\\
&=\vd_\theta^\top\vr
\label{eq:cumulative_obj_dual}
\end{align}
where the last equation is due to duality~(\citealt{wang2007dual}, Lemma~9; \citealt{puterman2014markov}, Sec.~6.9).

Although the term \emph{policy gradient} can mean any gradient of a policy in the literature, we use this term to specifically refer to the gradient of \cref{eq:cumulative_obj} throughout the paper. 
It is given by~\citep[Thm.1]{sutton2000policy}
\begin{align}
	\nabla_\theta J = \sum_{s} d_{\Scal,\theta}(s) \sum_a q_\theta(s, a) \nabla_\theta \pi_\theta(s,a)
    \label{eq:policy_gradient}
\end{align}
One can replace the $Q$-values with a critic in the policy gradient to obtain:\footnote{$\nabla^\phi_\theta J$ matches $\nabla_\theta J$ under certain technical assumptions, see~\citet[Thm.2]{sutton2000policy} for further discussion.}
\begin{align}
    \nabla^\phi_\theta J
    \defeq \sum_{s} d_{\Scal,\theta}(s) \sum_a q_\phi(s, a) \nabla_\theta \pi_\theta(s,a)
\label{eq:sutton_grad}
\end{align}
An \textit{actor-critic} (AC) method alternates between improving the policy (actor) using the critic, and estimating the policy's $Q$-values with a critic. 
AC methods are typically derived from \textit{policy iteration}~\citep{sutton2018reinforcement}, which alternates between policy evaluation and policy improvement.
\cref{eq:sutton_grad} can be used to update the actor in the policy improvement step.

\section{Unifying Classical Algorithms}

We consider a unified perspective on AC (and related) algorithms
that will allow us to better understand their relationships and 
ultimately characterize the key differences.
This will set the stage for the main contributions to follow,
although the perspectives are independently useful.

In particular, we
consider AC~\citep{sutton2018reinforcement} from 
two perspectives: (1) starting from the cumulative reward objective (objective perspective) and (2) from the policy gradient (gradient perspective). 
These two perspectives differ in 
where the policy's $Q$-values are approximated with a parametrized critic; in the first case, this is done in the policy objective (\cref{eq:cumulative_obj}) \emph{before} computing the gradient with respect to the policy parameters, while in the second case, the approximation appears in the gradient expression (\cref{eq:policy_gradient}) \emph{after} taking the gradient of the policy objective.

\subsection{Actor-Critic from the Objective Perspective}
\label{sec:actor_critic_obj}

Consider the two standard actor-critic objectives:
\begin{align}
\text{Actor:}\ 
\max_{\vtheta}\ J_\pi(\vtheta,\vphi)
&\defeq (1-\gamma)\vmu_0^\top \Pi_\theta \vq_\phi
\label{eq:actor_obj}\\
\text{Critic:}\ 
\min_{\vphi}\ J_q(\vtheta,\vphi)
&\defeq\frac{1}{2}
\|\vr + \gamma P\Pi_\theta \vq_\phi - \vq_\phi\|_{\vd}^2
\label{eq:q_obj}
\\
=\frac{1}{2}
&(\vr - \Psi_\theta \vq_\phi)^\top D 
(\vr - \Psi_\theta \vq_\phi)   \nonumber \\
\text{where}\qquad
\Psi_\theta
&\defeq I-\gamma P\Pi_\theta \nonumber
\end{align}
The \emph{actor objective} is an approximation of \cref{eq:cumulative_obj} using a parameterized critic $\vq_\phi$ in place of $\vq_\theta$. 
To learn a critic that estimates the policy's $Q$-values, the \emph{critic objective} minimizes the Bellman residual weighted by a state-action distribution $\vd$, which is assumed to have full support.
From a practical standpoint, $\vd$ is equivalent to a (fixed) replay buffer distribution from which the state-actions are sampled. 
An on-policy assumption is equivalent to setting $\vd = \vd_\theta$, the current policy's stationary distribution.
Given these objectives, the partial derivatives are\footnote{Unlike~\cref{eq:policy_gradient}, we use $\partial_\theta$ instead of $\nabla_\theta$ when the quantity in question depends on both $\vtheta$ and $\vphi$.}
\begin{align}
\partial_\theta J_\pi
&=(1-\gamma)\sum_s \mu_0(s)
\sum_a q_\phi(s,a)\nabla_\theta\pi_\theta(s,a)
\label{eq:PI_grad_cont}
\\
\partial_\phi J_q
&=H_\phi\Psi_\theta^\top D (\Psi_\theta \vq_\phi - \vr)
=-H_\phi\Psi_\theta^\top D \vdelta_{\theta,\phi} 
\label{eq:D2f2}
\end{align}
where $\vdelta_{\theta,\phi}:=\vr - \Psi_\theta\vq_\phi$ is the \emph{residual} of the critic. 
From this objective perspective, actor-critic (which we refer to as Actor$_o$-Critic) alternates between ascent on the actor and descent on the critic using their respective gradients\footnote{For notational simplicity, we use the same learning rate $\alpha$ for both the actor and the critic even though this is not needed for any of our conclusions.}:
\begin{alignat}{2}
\text{Actor$_o$:}\qquad
&&\vtheta
&\leftarrow \vtheta + \alpha\ \partial_\theta J_\pi
\label{eq:ac_a_update}
\\
\text{Critic:}\qquad
&&\vphi
&\leftarrow \vphi - \alpha\ \partial_\phi J_q
\label{eq:ac_q_update}
\end{alignat}
As an example, Soft Actor-Critic~\citep{haarnoja2018soft} can be considered as using variants of these updates (with entropy regularization, as shown in 
\ifmaintextonly
Appendix C).
\else
\cref{app:soft_ac}).
\fi

\subsection{Actor-Critic from the Gradient Perspective}\label{sec:ac_grad}
Alternatively, one can consider $\nabla_\theta^\phi J$ in \cref{eq:sutton_grad}
which applies the critic \emph{after} the gradient is derived. 
A2C~\citep{mnih2016asynchronous}
is an example of such an approach.
A key observation is that
$\nabla^\phi_\theta J$ differs from $\partial_\theta J_\pi$ (\cref{eq:PI_grad_cont}) in
terms of the state distribution they consider:
$\partial_\theta J_\pi$ uses the initial state distribution, 
whereas
$\nabla_\theta^\phi J$ uses the policy's stationary distribution. 
Therefore, using
$\nabla^\phi_\theta J$, one can define an alternative actor update
\begin{alignat}{2}
\text{Actor$_g$:}\qquad
&&\vtheta
&\leftarrow \vtheta + \alpha\ \nabla^\phi_\theta J.
\label{eq:ac_grad_a_update}
\end{alignat}
To distinguish this from Actor$_o$-Critic, we refer to the updates 
(\ref{eq:ac_q_update})--(\ref{eq:ac_grad_a_update}) collectively as Actor$_g$-Critic 
(since
the actor update is derived from a gradient perspective). In the 
literature, AC typically refers to what we call Actor$_g$-Critic \citep{sutton2000policy, konda2000actor}. 

\subsection{Unifying and Relating Classical Algorithms}
\label{sec:algo_relations}

Given these two perspectives, we can now show how classical algorithms
can simply be interpreted as alternative
interplays between the actor and critic updates;
see \cref{tab:algo_relations}.

\begin{table}[t]
\setlength{\tabcolsep}{3pt}
\centering
\begin{tabular}{c|c|c}
\hline
\diagbox{$\vtheta$ Update}{$\vphi$ Update}
& $\vphi\leftarrow \vphi-\partial_\phi J_q$ &  $\vq_\phi\leftarrow \vq_\theta$ \\
\hline
$\vtheta\leftarrow\vtheta+\partial_\theta J_\pi$ 
& Actor$_o$-Critic & Policy Gradient$_o$\\
\hline
$\vtheta\leftarrow\vtheta+\nabla^\phi_\theta J$ 
& Actor$_g$-Critic & Policy Gradient$_g$\\
\hline
$\vtheta\leftarrow\vtheta^*(\vq_\phi)$ 
& Q-Learning$^\dagger$ & Policy Iteration\\
\hline
\end{tabular}
\caption{Algorithms based on update rules. Note that $\vtheta\leftarrow\vtheta^*(\vq_\phi)$ is finding the greedy policy w.r.t.\ current $\vq_\phi$ (Policy Improvement), while $\vq_\phi\leftarrow \vq_\theta$ is learning the on-policy values of current $\vpi_\theta$ (Policy Evaluation).
$^\dagger$ indicates that the semi-gradient is used for the critic.
} 
\label{tab:algo_relations}
\end{table}

\textbf{Policy Iteration},
for example,
alternates between policy evaluation and policy improvement.
\emph{Policy evaluation} corresponds to fully optimizing the critic 
to becoming the policy's $Q$-values: $\vq_\phi\leftarrow \vq_\theta$. 
\emph{Policy improvement} 
updates the policy to be greedy with respect to its $Q$-values.
Since the critic is fully optimized, 
it is equivalent to a policy that is greedy with respect to the current critic: $\vtheta\leftarrow\vtheta^*(\vq_\phi)$. 
As noted in \cref{sec:background}, AC methods are typically derived from policy iteration. 
What is often overlooked, however, is that
Actor$_o$-Critic and Actor$_g$-Critic present
distinct instances of this general framework. 

\textbf{Policy Gradient$_g$}.
When the critic is fully optimized for every actor step (i.e., policy evaluation with infinite critic capacity $\vq_\phi\leftarrow \vq_\theta$), we have $\nabla_\theta^\phi J=\nabla_\theta J$, 
which
corresponds to the classical policy gradient method (\cref{eq:policy_gradient}). 
Note that this is different from \textbf{Policy Gradient$_o$}, which is defined to be $\partial_\theta J_\pi$ (\cref{eq:PI_grad_cont}) with $\vq_\phi=\vq_\theta$ (similar to Actor$_o$ from the objective perspective). 
The key difference is that Policy Gradient$_g$ uses the the on-policy distribution $d_{\Scal,\theta}(s)$ to weight the states, whereas
Policy Gradient$_o$ uses the initial state distribution $\mu_0(s)$. 

\textbf{Q-Learning}. 
The critic gradient in \cref{eq:D2f2} is the gradient of the expected squared Bellman residual, 
which involves double-sampling \citep{baird1995residual}. 
Therefore,
a \textit{semi-gradient}~\citep[Sec.9.3]{sutton2018reinforcement} is used
in practice.
If the critic is directly parametrized (i.e., $\vq_\phi=\vphi=\vq \in \mathbb{R}^{|\mathcal{S}||\mathcal{A}|\times 1}$), the critic update becomes
\begin{align}
\vq
\leftarrow\vq-\alpha\ \partial_\phi^{\text{semi}} J_q
&=(I-\alpha D)\vq
+\alpha D \vq'
\label{eq:q_learning_update}
\\
\text{where}\quad 
\partial_\phi^{\text{semi}} J_q
&\defeq -D\vdelta_{\theta,\phi}
\label{eq:semi_critic_grad}
\end{align}
and $\vq'\defeq\vr+\gamma P\Pi_\theta\vq$ is the \emph{target value}, whose gradient is ignored. 
When the policy is fully optimized w.r.t.\ 
$\vq$ 
(last row of \cref{tab:algo_relations}), 
$\gamma P\Pi_\theta\vq$ will choose the maximum next-state value, 
making $\vq'$ the usual Q-Learning target. 
When $D=D_\theta$ in \cref{eq:q_learning_update},
$\vq$ is updated according to on-policy experience,
yielding the on-policy Q-Learning algorithm.

\section{\longname}\label{sec:res_ac}

Based on the unified perspective, we now present one of our key contributions, which is a characterization of the gap between AC and PG methods,
both in terms
of the objectives~(\cref{sec:gap_obj}) and the gradients~(\cref{sec:grad_perspective}).
Then we propose a practical algorithm to reduce the gap/bias introduced by the critic in \cref{sec:updaterules}.

To begin, note that the
actor objective (\ref{eq:actor_obj})
differs from the policy objective (\ref{eq:cumulative_obj})
in that the critic value function $\vq_\phi$ is independent of the policy. 
Therefore, there is a discrepancy between the policy gradient $\nabla_\theta J$ and the partial derivative $\partial_\theta J_\pi$ of the actor objective. 

\subsection{Objective Perspective}\label{sec:obj_perspective}
\label{sec:gap_obj}

To characterize the gap from the objective perspective, consider
the difference between the policy objective (\ref{eq:cumulative_obj}) and the actor objective (\ref{eq:actor_obj}) in $\text{Actor}_o$-Critic. 
Using the dual formulation of the policy objective (\ref{eq:cumulative_obj_dual}), 
the difference is: 
\begin{align}
\vd_\theta^\top\vr
\!-\!(1\!-\!\gamma)\vmu_0^\top \Pi_\theta \vq_\phi
\!=\!\vd_\theta^\top\vr - \vd_\theta^\top\Psi_\theta\vq_\phi
\!=\!\vd^\top_\theta\vdelta_{\theta,\phi}
\label{eq:objective_gap}
\end{align}
where the first equality replaces $(1-\gamma)\vmu_0^\top \Pi_\theta$ with $\vd_\theta^\top\Psi_\theta$ using \cref{eq:d_recursion_matrix}.
Therefore, the difference is
$\vd^\top_\theta\vdelta_{\theta,\phi}$, which is 
the inner product between the stationary distribution of $\vpi_\theta$ 
and the on-policy residual of the critic $\vq_\phi$ under $\vpi_\theta$. 
This also implies that 
the difference between the Actor$_o$-gradient, $\partial_\theta J_\pi$, and the 
policy gradient, $\nabla_\theta J$,
is exactly equal to $\partial_\theta(\vd_\theta^\top\vdelta_{\theta,\phi})$, which we analyze below.

Note that by the product rule,
\begin{align}
\partial_\theta(\vd_\theta^\top\vdelta_{\theta,\phi})
&=\partial_\theta((\vd'_\theta)^\top\vdelta_{\theta,\phi}) + \partial_\theta(\vd_\theta^\top\vdelta'_{\theta,\phi})
\label{eq:d_delta_product_grad}
\end{align}
where $\vd'_\theta$ and $\vdelta'_{\theta,\phi}$ are treated as being independent of the policy parameter $\vtheta$ (i.e., without computing gradients). 
Then the policy gradient can be expressed as
\begin{align}
\nabla_\theta J=
\partial_\theta J_{\pi}
+\partial_\theta ((\vd'_\theta)^\top\vdelta_{\theta,\phi})
+\partial_\theta (\vd^\top_\theta\vdelta'_{\theta,\phi})
\label{eq:three_term_obj}
\end{align}

The first term is given by \cref{eq:PI_grad_cont}, and the second term is
\begin{align}
&\EE_{(S,A)\sim d_\theta}[\partial_\theta\delta_{\theta,\phi}(S,A)]
\nonumber\\
&=\EE_{(S,A)\sim d_\theta,\widetilde{S}\sim P_{SA}}
\left[\gamma\partial_\theta
\sum_{\widetilde{A}}\pi_\theta(\widetilde{S},\widetilde{A})q_\phi(\widetilde{S},\widetilde{A})\right]
\label{eq:third_term_grad}
\end{align}
To further simplify, one can use the $\log$ derivative trick (i.e., $\nabla_\theta \pi_\theta = \pi_\theta \nabla_\theta \log \pi_\theta$) so that
\begin{align}
\partial_\theta J_\pi 
&=(1-\gamma)\EE_{S\sim\mu_0,A\sim\pi_S}
[q_\phi(S,A) \nabla_\theta \log \pi_\theta(S,A)]
\nonumber\\
\partial_\theta &((\vd'_\theta)^\top\vdelta_{\theta,\phi})
=\gamma\EE_{(S,A)\sim d_\theta,\widetilde{S}\sim P_{SA},\widetilde{A}\sim\pi_{\widetilde{S}}}
[
\nonumber\\
&\qquad\qquad\qquad\qquad 
q_\phi(\widetilde{S},\widetilde{A}) \nabla_\theta \log \pi_\theta(\widetilde{S},\widetilde{A})]
\end{align}
Note that these two terms can be combined, using the recursive definition of $\vd_\theta$ (\cref{eq:d_recursion_matrix}), as
\begin{align}
&\partial_\theta 
[J_\pi
+(\vd'_\theta)^\top\vdelta_{\theta,\phi}]
\nonumber\\
&=\EE_{(S,A)\sim d_\theta}
[q_\phi(S,A) \nabla_\theta \log\pi_\theta(S,A)]
\nonumber\\
&=\sum_s d_\theta(s)\sum_a q_\phi(s,a)\nabla_\theta\pi_\theta(s,a)
\label{eq:first_and_third_term_grad}
\end{align}
which is equivalent to $\nabla_\theta^\phi J$ in \cref{eq:sutton_grad}.
This is somewhat surprising. 
Now one
can see that $\nabla_\theta^\phi J$ is in fact maximizing $J_\pi+(\vd'_\theta)^\top\vdelta_{\theta,\phi}$,
which is nearly identical to the policy objective except that
it ignores the last term of (\ref{eq:three_term_obj}), $\partial_\theta (\vd^\top_\theta\vdelta'_{\theta,\phi})$ (i.e., the dependence of $\vd_\theta$ on $\vtheta$).
The following theorem summarizes the gap between policy gradient and actor gradients.
\begin{theorem}
\label{thm:pg_ac_gap}
The gap between the policy gradient $\nabla_\theta J$ and $\partial_\theta J_{\pi}$ used in the Actor$_o$ update is given by
\begin{align*}
\nabla_\theta J
-\partial_\theta J_{\pi}
&=
\partial_\theta 
\EE_{(S,A)\sim d_\theta}
[\delta_{\theta,\phi}(S,A)]
\end{align*}
and the gap between the policy gradient $\nabla_\theta J$ and $\nabla_\theta^\phi J$ used in the Actor$_g$ update is given by
\begin{align*}
\nabla_\theta J
-\nabla_\theta^\phi J
&=
\partial_\theta 
\EE_{(S,A)\sim d_\theta}
[\delta'_{\theta,\phi}(S,A)].
\end{align*}
\end{theorem}

We will discuss how to estimate $\partial_\theta (\vd^\top_\theta\vdelta'_{\theta,\phi})$ in \cref{sec:updaterules}.


\subsection{Gradient Perspective}
\label{sec:grad_perspective}

Additional insight is gained by considering the difference between AC and PG
from the gradient perspective.
This section provides a different and possibly more direct way
to correct $\nabla_\theta^\phi J$, where 
the critic replaces the on-policy values \emph{after} the policy gradient is derived. 
Recall the policy gradient and $\nabla_\theta^\phi J$,
and observe that they
can be rewritten 
respectively
as (see 
\ifmaintextonly
Appendix A):
\else
\cref{app:notation}):
\fi
\begin{align}
\nabla_\theta J
\!=\!H_\theta\Delta(\Xi^\top\vd_{\Scal,\theta})\vq_\theta
,\ 
\nabla^\phi_\theta J
\!=\!H_\theta\Delta(\Xi^\top\vd_{\Scal,\theta})\vq_\phi
\label{eq:PG_grad}
\end{align}
Clearly, $\nabla^\phi_\theta J\neq\nabla_\theta J$ unless $\vq_\phi$ satisfies 
specific conditions~\citep[Thm.2]{sutton2000policy}.
Their difference is:
\begin{align}
\nabla_\theta J
-\nabla^\phi_\theta J
&=H_\theta\Delta(\Xi^\top\vd_{\Scal,\theta})(\vq_\theta-\vq_\phi)
\\
&=H_\theta\Delta(\Xi^\top\vd_{\Scal,\theta})\Psi_\theta^{-1}(\vr-\Psi_\theta\vq_\phi)
\\
&=H_\theta\Delta(\Xi^\top\vd_{\Scal,\theta})\Psi_\theta^{-1}\vdelta_{\theta,\phi}
\label{eq:grad_cancel_q}
\end{align}

To further simplify, consider the following theorem.

\begin{restatable}{theorem}{stationary}[Stationary distribution derivative]
\label{thm:sdd}
Let the derivative matrix $\Upsilon$ of the stationary distribution w.r.t.\ the policy parameters to be $\Upsilon_{i,sa}=\frac{\partial d_\theta(s,a)}{\partial\theta_i}$. 
Then
\begin{align}
\Upsilon
=H_\theta\Delta(\Xi^\top \vd_{\Scal,\theta})\Psi^{-1}_\theta.
\nonumber
\end{align}
\end{restatable}
All proofs can be found in  
\ifmaintextonly
Appendix B.
\else
\cref{app:proofs}.
\fi
Using this theorem, one can see that \cref{eq:grad_cancel_q} is
in fact
\begin{align}
\nabla_\theta J
-\nabla^\phi_\theta J
=\Upsilon\vdelta_{\theta,\phi}
=\partial_\theta (\vd_\theta^\top\vdelta'_{\theta,\phi}),
\label{eq:pg_actor_g_diff}
\end{align}
which provides an alternative way to prove the gap between PG and AC as shown in \cref{thm:pg_ac_gap}.

\cref{thm:sdd} in itself can be useful in some other contexts. 
A related result can be found in \citet[Eq.(3.5)]{morimura2010derivatives},
which is a recursive form for the derivative of the log stationary distribution.
Our theorem here provides a direct and explicit form for the derivative of the distribution.

\subsection{{\longname} Update Rules}
\label{sec:updaterules}

The key insight from both \cref{sec:obj_perspective,sec:grad_perspective} is that bridging the gap between AC and PG requires the computation of $\partial_\theta (\vd_\theta^\top\vdelta'_{\theta,\phi})$. 
Our next main contribution is to develop a practical strategy to estimate this gap, which will reduce the bias introduced by the critic and bring the actor update in AC closer to the true policy gradient.
This results in a new AC framework we call \textit{\longname}.

To develop a practical estimator,
first note that 
$\vd_\theta^\top\vdelta'_{\theta,\phi}$
can be treated as a dual policy objective (\cref{eq:cumulative_obj_dual}), where the environment's reward $\vr$ is replaced by the residual $\vdelta'_{\theta,\phi}$ of the critic $\vq_\phi$. 
The corresponding primal objective is
\begin{align}
\max_{\vtheta}\ J_{\delta}(\vtheta)
&\defeq (1-\gamma)\vmu_0^\top \Pi_\theta \vw_\theta
\label{eq:actor_obj_third_term}
\end{align}
where $\vw_\theta=\sum_{i=0}^\infty(\gamma P\Pi_\theta)^i\vdelta'_{\theta,\phi}$ is the on-policy $Q$-value associated with the residual reward.
The gradient of (\ref{eq:actor_obj_third_term}) is precisely the desired correction, $\partial_\theta (\vd_\theta^\top\vdelta'_{\theta,\phi})$.
Computing its gradient requires $\vw_\theta$, which can be approximated by introducing a \emph{residual-critic} (or \emph{res-critic}) $\vw_\psi$ with parameters $\vpsi$. 
Concretely, the res-critic solves the following problem:
\begin{align}
\min_{\vpsi}\ J_w(\vpsi;\vtheta,\vphi)
&\defeq\frac{1}{2}
\|\vdelta'_{\theta,\phi} + \gamma P\Pi_\theta \vw_\psi - \vw_\psi\|_{\vd}^2
\label{eq:res_q_obj}
\end{align}

Once we have a relatively accurate res-critic $\vw_\psi$, we apply the PG$_g$ method (see \cref{tab:algo_relations}) and use the following to approximate $\partial_\theta (\vd_\theta^\top\vdelta'_{\theta,\phi})$
\begin{align}
    \nabla^\psi_\theta J_\delta
    \defeq \sum_{s} d_\theta(s) \sum_a w_\psi(s, a) \partial_\theta \pi_\theta(s,a)
\label{eq:sutton_res_grad}
\end{align}
Combining these update rules for the actor and res-critic with the standard AC update rules results in our {\longname}~({\shortname}) framework, which can be summarized as follows:
\begin{alignat}{2}
\text{Actor:}\qquad
&&\vtheta
&\leftarrow \vtheta + \alpha\ 
(\nabla_\theta^\phi J + \nabla_\theta^\psi J_\delta)
\label{eq:res_ac_a_update}
\\
\text{Critic:}\qquad
&&\vphi
&\leftarrow \vphi - \alpha\ \partial_\phi J_q
\label{eq:res_ac_q_update}
\\
\text{Res-Critic:}\qquad
&&\vpsi
&\leftarrow \vpsi - \alpha\ \partial_\psi J_{w}
\label{eq:res_ac_res_q_update}
\end{alignat}

To understand the correction term $\nabla_\theta^\psi J_\delta$ intuitively, note that the residual reward $\vdelta'_{\theta,\phi}$ is signed. 
For an $(s,a)$ with $\delta'_{\theta,\phi}(s,a)>0$, we have $q_\phi(s,a)<r(s,a)+\gamma\EE[q_\phi(\widetilde{s},\widetilde{a})]$, and the agent is incentivized to visit the underestimated region.
On the other hand, if $\delta'_{\theta,\phi}(s,a)<0$, the agent is discouraged to visit the overestimated location.

{\shortname} is a generic framework that can be combined with different AC-based algorithms.
In 
\ifmaintextonly
Appendix C
\else
\cref{app:soft_ac}
\fi
and in the experiments, we show that Soft Actor-Critic (SAC)~\citep{haarnoja2018soft} can be enhanced with a res-critic and the resultant Res-SAC method improves over the original SAC.

\section{Stackelberg Actor-Critic as a Special Case}

Before proceeding to an experimental evaluation of the new {\shortname} framework,
we first present another, somewhat surprising finding that the 
results above are also consistent with the characterization of AC as a Stackelberg game.
That is, previously we focused on correcting the actor's gradient in both Actor$_o$-Critic and Actor$_g$-Critic to obtain the true policy gradient,
whereas now we consider the interplay between actor and critic from a game-theoretic perspective. 
Here we will be able to
show that when treating AC as a Stackelberg game~\citep{sinha2017review}, the \textit{Stackelberg policy gradient} is in fact the true policy gradient under certain conditions. 
This will follow as a special case of the analysis in \cref{sec:gap_obj}.
Moreover, we show that even when the critic update is based on semi-gradient (i.e., with a fixed target), the Stackelberg policy gradient remains unbiased.

For this section we restrict the AC formulation by adding the following assumption. 
\begin{assumption}
\label{assume:stackelberg}
The critic is directly parametrized $\vq_\phi=\vphi=\vq$. 
The critic loss is weighted by the on-policy distribution $\vd=\vd_\theta$ in \cref{eq:q_obj}.
$\vmu_0$ and $\vpi_\theta$ have full support.
\end{assumption}

\subsection{Actor-Critic as a Stackelberg Game}

Actor-critic methods can be considered as a two-player general-sum game, where the actor and critic are the players and the objectives (\cref{eq:actor_obj,eq:q_obj}) are their respective utility/cost functions. 
More specifically, one can treat actor-critic as a \emph{Stackelberg game} in which there is a \emph{leader} who moves first and a \emph{follower} who moves subsequently~\citep{sinha2017review}.
By treating the actor as the leader,
Stackelberg Actor-Critic (Stack-AC) solves the following
\begin{align}
\max_{\vtheta} &\ \left\{J_\pi(\vtheta,\vq)
\middle|
\vq\in\argmin_{\widetilde{\vq}}\ J_q(\vtheta,\widetilde{\vq})
\right\}
\\
\min_{\vq} &\ J_q(\vtheta,\vq)
\end{align}
A key distinction from the original AC is that the actor is now aware of the critic's goal. 
Given that in the ideal case $\vq$ is implicitly a function of $\vtheta$ (i.e., policy evaluation), one may differentiate through $\vq$ to obtain the following \emph{Stackelberg gradient}\footnote{It is the total derivative from the implicit function theorem.}~\citep{fiez2020implicit}
\begin{align}
\vg_{S,\theta}
&\defeq\partial_\theta J_\pi
-(\partial_\theta\partial_q J_q)^\top(\partial_q^2 J_q)^{-1}(\partial_q J_\pi)
\label{eq:stack_grad_pi}
\end{align}

The second order derivative $\partial_q^2J_q$ can be computed, based on the critic objective \cref{eq:q_obj}, as
\begin{align}
\partial_q^2 J_q&=\Psi_\theta^\top D_\theta \Psi_\theta
\label{eq:D22f2}
\end{align}
which is invertible under \cref{assume:stackelberg}, 
hence the Stackelberg gradient is well-defined. 
However, it is unclear what this gradient achieves in the AC setting. 
We show, in the following theorem, that the Stackelberg gradient is in fact the policy gradient of the cumulative reward objective (\ref{eq:cumulative_obj}). 

\begin{restatable}{theorem}{stackelberg}
\label{thm:stack_pg}
Under \cref{assume:stackelberg},
\begin{equation}
\vg_{S,\theta}
=\partial_\theta J_\pi +\partial_\theta(\vd_\theta^\top\vdelta_{\theta,\phi})
=\nabla_\theta J.
\nonumber
\end{equation}
\end{restatable}

This indicates that, under some conditions, one can compute the gradient of the objective correction in \cref{eq:objective_gap} using $(\partial_\theta\partial_q J_q)^\top(\partial_q^2 J_q)^{-1}(\partial_q J_\pi)$.

\subsection{Semi-Gradient Extension}
As discussed earlier in \cref{sec:algo_relations}, it is common to use a semi-gradient update for the critic to address the double sampling issue.
This section shows that, surprisingly, the Stackelberg gradient remains the true policy gradient even when using semi-gradient for the critic.

From the semi-gradient (\ref{eq:semi_critic_grad}), the \emph{semi-Hessian} is
\begin{align}
(\partial_q^{\text{semi}})^2 J_q
=\partial_q (D_\theta(\vq-\vq'))
=D_\theta.
\label{eq:D22f2_semi}
\end{align}
Compared to \cref{eq:D22f2}, the derivative does not go through the next-state values 
so
$\Psi_\theta$ is now replaced by the identity matrix.
Additionally, the actor objective can be reformulated using \cref{eq:d_recursion_matrix} as
\begin{align}
J_\pi(\vtheta,\vq)
=\vd^\top_\theta\Psi_\theta\vq
=\vd_\theta^\top(\vq-\vq')
\label{eq:actor_obj_d_formula}
\end{align}
Thus the semi-derivative is
\begin{align}
\partial_q^{\text{semi}} J_\pi
=\vd_\theta
\end{align}
Then the Stackelberg gradient based on semi-critic-gradient is given by
\begin{align}
\vg_{S,\theta}^{\text{semi}}
&\!\defeq\!\partial_\theta J_\pi
\!-\!(\partial_\theta\partial_q^{\text{semi}} J_q)^\top\!
((\partial_q^{\text{semi}})^2 J_q)^{-1}\!
(\partial_q^{\text{semi}} J_\pi)
\label{eq:stack_semi_grad_pi}\\
&=\partial_\theta J_\pi
-(\partial_\theta(-D_\theta\vdelta_{\theta,\phi}))^\top
D_\theta^{-1}\vd_\theta
\\
&=\partial_\theta J_\pi
-(\partial_\theta(-D_\theta\vdelta_{\theta,\phi}))^\top
\onevec_{|\Scal||\Acal|}
\\
&=\partial_\theta J_\pi
+\partial_\theta(\vd_\theta^\top \vdelta_{\theta,\phi})
=\vg_{S,\theta}=\nabla_\theta J.
\end{align}
which again corresponds to the true policy gradient.

\subsection{Stack-AC Update Rules}\label{sec:stack_updates}

The update rules for Stack-AC are summarized as follows
\begin{alignat}{2}
\text{Actor:}\qquad
&&\vtheta
&\leftarrow \vtheta + \alpha\ 
\vg^{\text{semi}}_{S,\theta}
\label{eq:stack_ac_a_update}
\\
\text{Critic:}\qquad
&&\vq
&\leftarrow \vq - \alpha\ \partial^{\text{semi}}_q J_q
\label{eq:stack_ac_q_update}
\end{alignat}
where the Stackelberg gradient $\vg^{\text{semi}}_{S,\theta}$ can be approximated using sample-based estimate for each term in \cref{eq:stack_semi_grad_pi}.
Although \cref{eq:stack_semi_grad_pi} requires an inverse-Hessian-vector product and a Jacobian-vector product, they can be efficiently carried out or approximated using standard libraries~\citep{fiez2020implicit}.
Following \citet{fiez2020implicit}, we use a regularized version where $(\partial_q^{\text{semi}})^2 J_q$ is replaced by $(\partial_q^{\text{semi}})^2 J_q+\eta I$ with $\eta=0.5$ in our experiments (\cref{sec:experiments}). This can ensure invertibility and stabilize learning.

There is one caveat when estimating $\partial_\theta\partial_q^{\text{semi}} J_q$ in $\vg_{S,\theta}^{\text{semi}}$ from samples.
The critic objective can be written as $J_q=\frac{1}{2}\vd_\theta^\top\vdelta^2_{\theta,\phi}$. 
$\partial_q^{\text{semi}}J_q$ can be estimated using a batch $\Bcal$ of samples drawn from $\vd_\theta$.
However, $\partial_\theta\partial_q^{\text{semi}} J_q=\partial_\theta(-D_\theta\vdelta_{\theta,\phi})$ is now difficult to estimate because we cannot take derivative of $\vtheta$ through the batch $\Bcal$, which represents the derivative through $D_\theta$. 
As a result, a sample-based estimate of $\partial_\theta\partial_q^{\text{semi}}J_q$ is only estimating $\partial_\theta(-D'_\theta\vdelta_{\theta,\phi})$, where $D'_\theta$ is consider a fixed distribution unrelated to $\vtheta$.

To summarize, the Stackelberg policy gradient can close the gap between AC and PG under certain conditions, even with semi-gradient updates. 
Despite being biased when approximated using samples, we will show in our experiments that Stack-AC can work reasonably well in practice.

\section{Experiments}\label{sec:experiments}

\newcommand{\onethirdfigwidth}{0.32\textwidth}
\newcommand{\halffigwidth}{0.48\textwidth}
\newcommand{\onefourthfigwidth}{0.23\textwidth}

\begin{figure}[t]
\centering
\includegraphics[width=\halffigwidth]{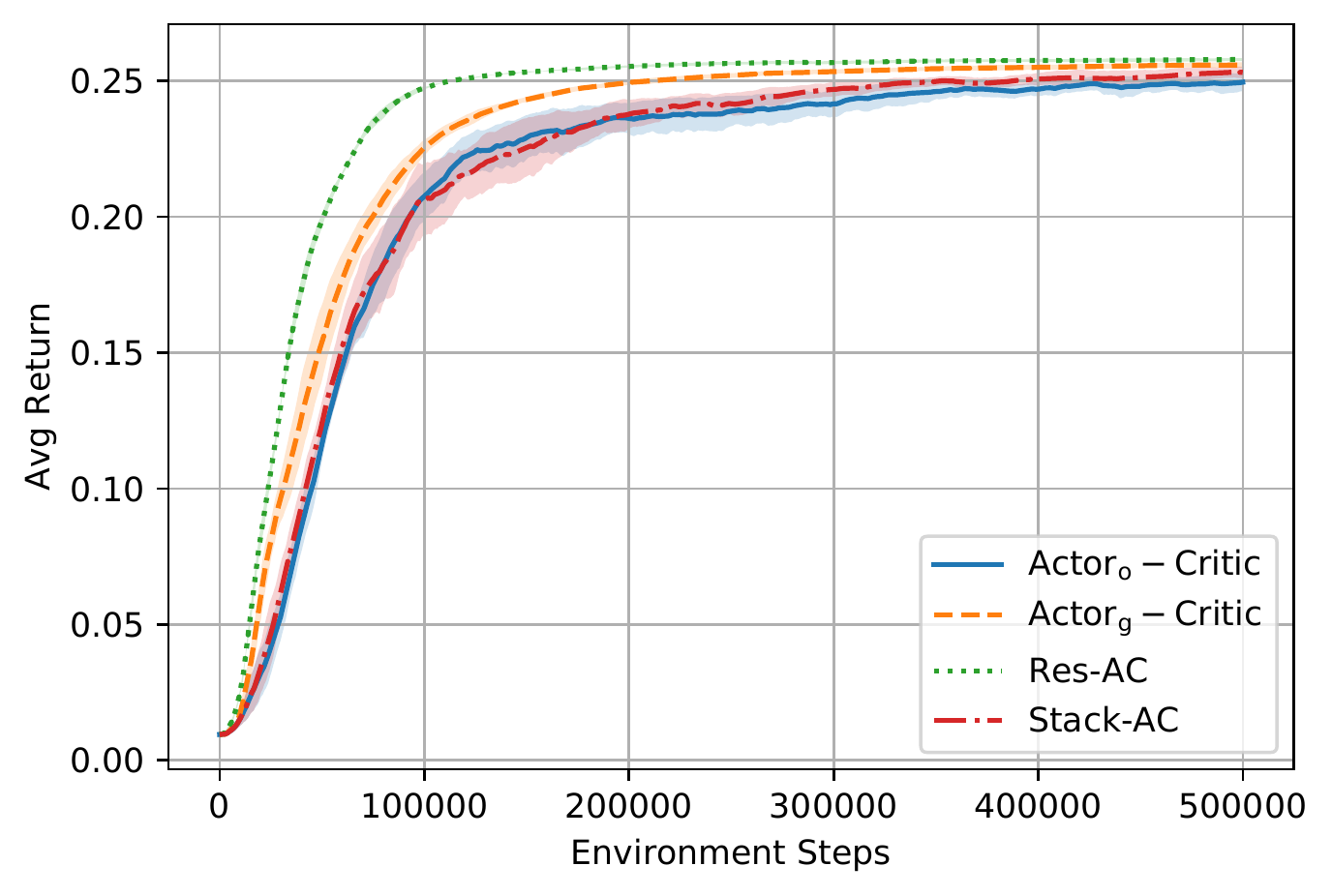}
\caption{Comparison of Actor$_o$-Critic, Actor$_g$-Critic, Stack-AC, and Res-AC on the FourRoom domain, plotting mean with one standard deviation as shaded region over 3 runs. Res-AC outperforms the other methods in both sample efficiency and final performance.}
\label{fig:four_room_sample}
\end{figure}

\begin{figure}[t]
\centering
\includegraphics[width=\halffigwidth]{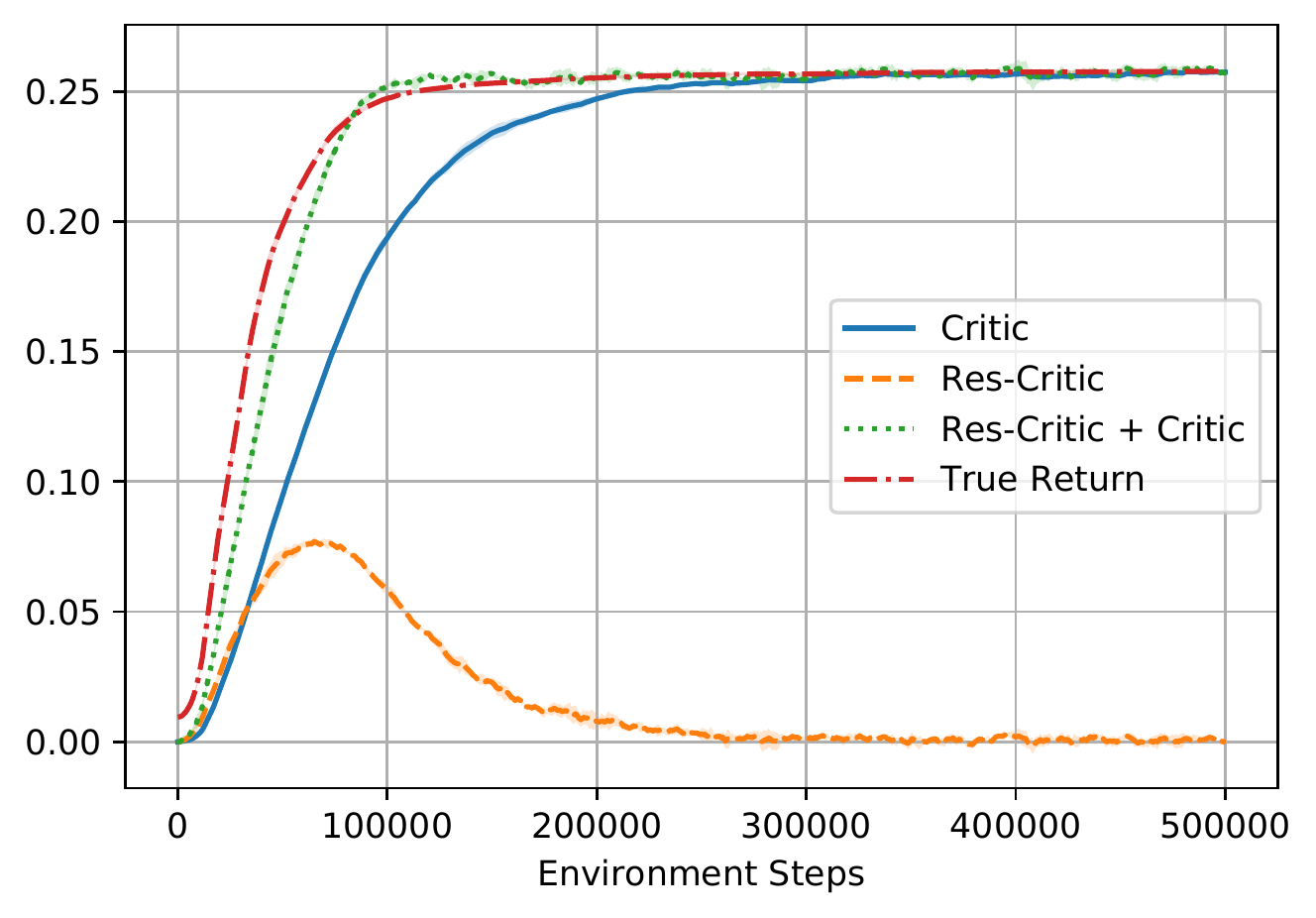}
\caption{The predicted returns of critic and residual critic for the FourRoom domain.
As training proceeds, the sum of the critic and res-critic closely approximates the true return.}
\label{fig:four_room_res_effect}
\end{figure}

The goal of our experiments is to test whether closing the gap between the actor's update in actor-critic (\cref{eq:PI_grad_cont}) and the policy gradient (\cref{eq:policy_gradient}) leads to improved sample efficiency and performance over actor-critic methods. To this end, we conduct experiments within both the FourRoom domain, an illustrative discrete action space environment (see
\ifmaintextonly
Appendix D),
\else
\cref{app:exp_details}),
\fi
and three continuous control environments: Pendulum-v0, Reacher-v2, and HalfCheetah-v2. The environments Reacher-v2 and HalfCheetah-v2 use the MuJoCo physics engine~\citep{todorov2012mujoco}.

On the FourRoom domain, we compare Actor$_o$-Critic and Actor$_g$-Critic to Res-AC and Stack-AC. 
For our continuous control experiments, we modify the actor update of Soft Actor-Critic (SAC) \citep{haarnoja2018soft}, a popular maximum entropy reinforcement learning method, using the updates given by Res-AC and Stack-AC. We refer to the resulting methods as Res-SAC and Stack-SAC, respectively, and we compare SAC to Res-SAC and Stack-SAC on the continuous control tasks. 
A complete derivation of the modified updates of Res-SAC and Stack-SAC can be found in 
\ifmaintextonly
Appendix C.
\else
\cref{app:soft_ac}.
\fi
Additional training details, including hyper-parameter settings and pseudocode, and additional experimental results are in 
\ifmaintextonly
Appendix D.
\else
\cref{app:exp_details}.
\fi

\begin{figure*}[t]
\centering
\includegraphics[width=\textwidth]{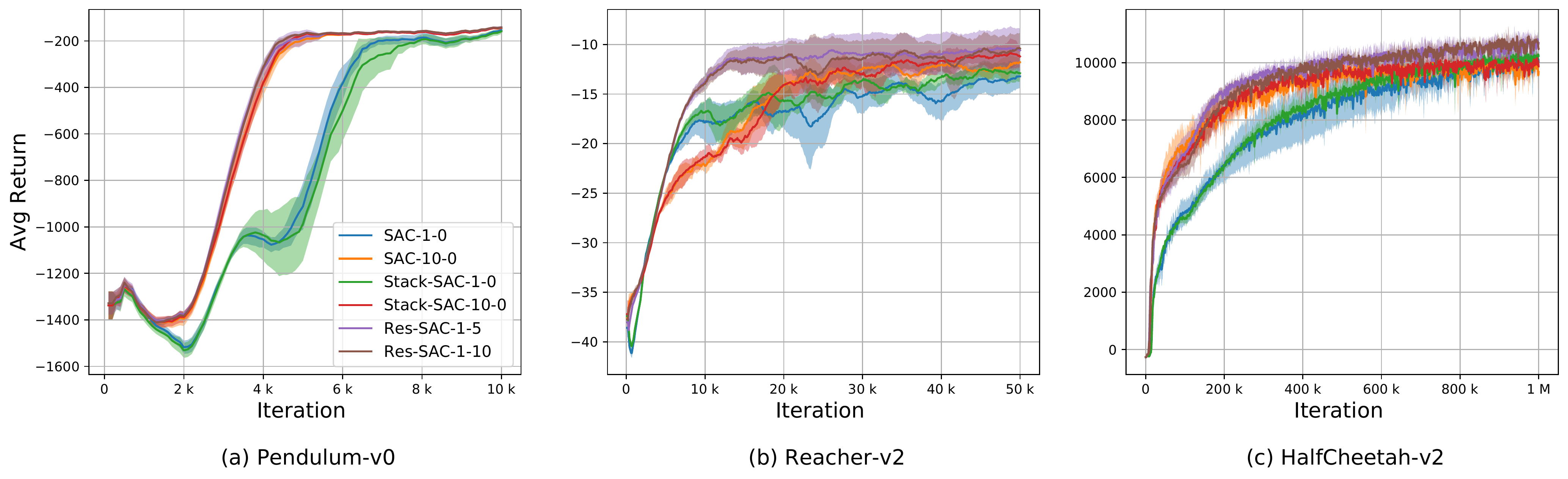}
\caption{Training curves on continuous control tasks. Res-AC consistently outperforms SAC and Stack-SAC on all three environments. Reporting mean with one standard deviation as shaded region over 5 runs.
Each iteration on the $x$-axis corresponds to 10 environment steps.}
\label{fig:mujoco}
\end{figure*}

\subsection{Tabular Experiment}\label{sec:fourroom}

In the FourRoom domain, we use a tabular parameterization $\vq$ for the critic and a softmax tabular actor $\vpi$. 
All our AC methods collect data from the environment and compute updates for the actor and critic~(and res-critic for Res-AC) using this data. 
We train each algorithm with three different random seeds, and we plot the return of the current policy after each episode, where every episode has a fixed length of 300 environment steps (\cref{fig:four_room_sample}). 
The curves correspond to the mean return and the shaded region to one standard deviation 
over the three trials. 

Res-AC enjoys improved sample-efficiency as well as final performance when compared to all other methods. 
Stack-AC and Actor$_o$-Critic achieve a lower final performance than Res-SAC, and they perform comparably to each other. 
Actor$_g$-Critic achieves a similar final return as Res-AC, but requires over 150,000 additional environment steps. 
The relatively poor performance of Stack-AC is not surprising since the sample-based estimate of $\partial_\theta\partial_q^{\text{semi}} J_q$ is inaccurate, as discussed in \cref{sec:stack_updates}. 
In contrast, even though {\shortname} introduces an additional problem of using a res-critic to learn the on-policy return with $\vdelta'_{\theta,\phi}$ as reward, the res-critic significantly accelerates the improvement of the actor. 

To gain further understanding of the residual critic, \cref{fig:four_room_res_effect} plots the predictions of the critic and res-critic for Res-AC on the FourRoom domain.
The sum of the critic and res-critic approximates the true return better than the critic alone throughout training.
This shows that the res-critic can correct the bias introduced by the critic in policy gradient empirically.

\subsection{Continuous Control Experiments}\label{sec:contexp}

We compare SAC to Res-SAC and Stack-AC on Pendulum-v0, Reacher-v2, and HalfCheetah-v2. 
We additionally introduce an update schedule for an algorithm $Algo$ labelled as ``$Algo$-$x$-$y$'', where $x$ and $y$ are the number of gradient updates applied to the critic and res-critic, respectively, for each actor gradient update. 
For example, SAC-10-0 refers to SAC with 10 critic gradient updates per actor update. (Note that the number of res-critic updates here is 0 since SAC does not use a res-critic.) In the original SAC algorithm, only one critic update is performed per actor update. Our decision to optionally perform multiple gradient updates for the critic / res-critic is guided by the intuition that a more accurate critic / res-critic would benefit the gradient updates to the policy.  

For SAC, Stack-SAC, and Res-SAC, we use the hyper-parameters used in \citet{haarnoja2018soft2} on all environments. 
Stack-SAC includes an additional regularization parameter $\eta$, introduced in \cref{sec:stack_updates}, which we set to 0.5 for all experiments. 
Even though the theoretical properties of Stackelberg gradient may not hold for continuous environments (\cref{thm:stack_pg}), we can still implement it for our experiments and investigate its performance. 
The performance of Res-SAC is dependent on clipping the residual that is used as the reward for the res-critic. Concretely, for a clip value $c>0$, the reward for updating the res-critic is computed as $\text{clip}(\vdelta_{\theta,\phi}, -c, c)$. 
Without clipping, large values of the critic's residual can make training the res-critic unstable. 
Additional discussion on the choice of clipping value can be found in 
\ifmaintextonly
Appendix D.2.2.
\else
\cref{app:clip_analysis}.
\fi

The results in \cref{fig:mujoco} show that Res-SAC is consistently more sample-efficient than the other methods and even achieves higher asymptotic performance on Reacher-v2. 
Applying multiple critic updates per actor-update significantly improves the performance of SAC, as SAC-10-0 consistently outperforms SAC-1-0. 
The better performance of Res-SAC over SAC-10-0 suggests that bringing the actor update closer to the true policy gradient in theory translates to empirical benefits. 
In contrast, Stack-SAC performs comparably to SAC -- there is not a clear benefit to using the Stack-SAC actor update rule over the standard SAC actor update. 
This is understandable as \cref{assume:stackelberg} does not hold when the critic uses non-linear function approximation.

\section{Related Work}

\textbf{Actor-critic and policy gradient}.
We review prior works which analyze the relationship between policy gradient and actor-critic and use policy iteration to derive AC methods. 
Actor-critic is typically derived following our Actor$_g$-Critic derivation \citep{mnih2016asynchronous,lillicrap2016continuous,liu2020off,degris2012off, peters2008natural}, and replacing the on-policy values in the policy gradient with a critic can retain the policy gradient under certain assumptions~\citep{konda2000actor,sutton2000policy}. 
In this paper, we show that replacing the on-policy values with a critic corresponds to a partial policy gradient (\cref{thm:pg_ac_gap}). 
We also introduce an alternative derivation of actor-critic from an objective perspective, which has been less explored in the literature, and existing algorithms can be better understood using this perspective. As an example, while the motivation and derivation for SAC \citep{haarnoja2018soft} is based on policy iteration, its policy improvement step of minimizing a KL divergence objective per state, can be explained by the actor objective in our Actor$_o$-Critic framework (see 
\ifmaintextonly
Appendix C
\else
\cref{app:soft_ac}
\fi
for more details).

Several recent works \citep{GMR20, pgqcombine, pgequiv, bridge} demonstrate 
connections between policy and value based methods, which also apply to 
actor-critic methods. Other advancements such as TRPO~\citep{schulman2015trust}, GAE~\citep{schulman2016high}, and TD3~\citep{fujimoto2018addressing} can be considered as orthogonal to our analysis, and can be integrated into our Res-AC and Stack-AC updates. 

\textbf{Learning with Bellman residual}.
Our Res-AC approach uses the the \emph{residual of the critic} to facilitate learning. 
It is slightly different from the \emph{advantage function}~\citep{schulman2016high}: the former is an approximation
error of the critic (which would be zero for the perfect critic) while the latter represents relative gain of an action (which would not be zero unless all actions in the same state have the same value).
GTD2 and TDC~\citep{sutton2009fast,maei2009convergent} learn a form of the residual and use it to update the policy in the linear function approximation regime. In contrast, Res-AC learns a res-critic which is a value function of the residual.
\citet{sun2011incremental} showed that a value function of the residual can be used as an ideal feature/basis to assist learning the reward value function when using linear function approximation. Our analysis is more general, as we show that learning the critic residual value function can be used to reconstruct the true policy gradient for arbitrary policy parametrization.

The residual (or TD error) is used in some other contexts to facilitate training RL agents such as prioritizing which data to sample from a replay buffer to perform updates~\citep{schaul2016prioritized, van2013planning}, or estimating the variance of the return~\citep{sherstan2018comparing}. 
Res-AC is a more direct approach of leveraging the residual to improve the policy, since the res-critic is used in the actor update.
Additionally, \citet{dabney2020distributional} showed that dopamine neurons can respond to the prediction error differently, suggesting that the residual of the value function plays an important role biologically.

\textbf{Game-theoretic perspective.}
The concept of a differential Stackelberg equilibrium was proposed in \citet{fiez2020implicit} with a focus on the
convergence dynamics for learning Stackelberg games. 
A game theoretic framework for model based RL was proposed in \citet{rajeswaran2020game}, where the authors consider a Stackelberg game between an actor maximizing rewards and an agent learning an explicit model of the environment. 
It is computationally expensive as it requires solving actor/model to the optimum in every iteration.
Our Stack-AC, on the other hand, models the actor and critic as the players and shows connections between Stackelberg gradient and true policy gradient, even when using semi-gradient for the critic.
\citet{sinha2017review} provides a more comprehensive review on general bi-level optimization.

\section{Conclusion}

In this work, we characterize the gap between actor-critic and policy gradient methods. By defining the objective functions for the actor and the critic, we elucidate the connections between several classic RL algorithms. Our theoretical results identify the gap between AC and PG from both objective and gradient perspectives, and we propose {\shortname}, which closes this gap. Additionally, by viewing AC as a Stackelberg game, we show that the Stackelberg policy gradient is the true policy gradient under certain conditions. An empirical study on tabular and continuous environments illustrates that applying {\shortname} modifications to update rules of actor-critic methods improves sample efficiency and performance. 
Investigating the convergence guarantees of {\shortname} and developing Stack-AC methods where the critic is the leader are exciting directions for future work.

\section{Acknowledgements}

We would like to thank Robert Dadashi, Yundi Qian and anonymous reviewers
for constructive feedback.
This work is partially supported by NSERC, Amii, a Canada CIFAR AI Chair, an NSF Graduate Research Fellowship, and the Stanford Knight Hennessy Fellowship.

\bibliography{ref}
\bibliographystyle{apalike}

\clearpage
\newpage
\appendix
\onecolumn

\begin{center}
{\huge Appendix}
\end{center}

\section{Notations}
\label{app:notation}
For better understanding, this section provides notation conversions for some of the key concepts in the main text beyond matrix-vector notation. 

\begin{itemize}
\item $Q$-value
\begin{align*}
\vq_\theta
&=\sum_{i=0}^{\infty}(\gamma P\Pi_\theta)^i\vr
\\
q_\theta(s,a)
&=\EE_{A_i\sim\pi_\theta(S_i,\cdot),S_{i+1}\sim P(\cdot|S_i,A_i)}
\left[\sum_{i=0}^\infty\gamma^i r(S_i,A_i)\middle|
S_0=s,A_0=a\right]
\end{align*}
\item The stationary distribution's recursive form (\ref{eq:d_recursion_matrix})
\begin{align*}
\vd_\theta 
&= (1-\gamma)\Pi_\theta^\top\vmu_0
+\gamma \Pi_\theta^\top P^\top \vd_\theta
\\
d_\theta(s,a)
&=(1-\gamma)\mu_0(s)\pi_\theta(s,a)
+\gamma\sum_{\widetilde{s},\widetilde{a}}
d_\theta(\widetilde{s},\widetilde{a})P(s|\widetilde{s},\widetilde{a})
\pi_\theta(s,a)
\end{align*}
\item The cumulative reward objective (\ref{eq:cumulative_obj}), actor objective (\ref{eq:actor_obj}) and critic objective (\ref{eq:q_obj}):
\begin{align*}
J(\vtheta)
&=(1-\gamma)
\vmu_0^\top\Pi_\theta\sum_{i=0}^\infty(\gamma P\Pi_\theta)^i\vr
&&=(1-\gamma)\EE_{S_0\sim\mu_0,A_i\sim\pi_\theta(S_i,\cdot),S_{i+1}\sim P(\cdot|S_i,A_i)}
\left[\sum_{i=0}^\infty\gamma^i r(S_i,A_i)\right]
\\
&=(1-\gamma)\vmu_0^\top\Pi_\theta\vq_\theta
&&=(1-\gamma)\EE_{S_0\sim\mu_0,A_0\sim\pi_\theta(S_0,\cdot)}
\left[q_\theta(S_0,A_0)\right]
\\
&=\vd_\theta^\top\vr
&&=\EE_{(S,A)\sim d_\theta}[r(S,A)]
\\
J_\pi(\vtheta,\vphi)
&= (1-\gamma)\vmu_0^\top \Pi_\theta \vq_\phi
&&=(1-\gamma)\EE_{S_0\sim\mu_0,A_0\sim\pi_\theta(S_0,\cdot)}
\left[q_\phi(S_0,A_0)\right]
\\
J_q(\vtheta,\vphi)
&=\frac{1}{2}
\|\vr + \gamma P\Pi_\theta \vq_\phi - \vq_\phi\|_{\vd}^2
&&=\frac{1}{2}
\EE_{(S,A)\sim d}\left[\left(r(S,A)
+\gamma
\EE_{S'\sim P(\cdot|S,A),A'\sim\pi_\theta(S',\cdot)}
[q_\phi(S',A')]
-q_\phi(S,A)
\right)^2\right]
\\
\end{align*}
\item The policy gradient (\ref{eq:policy_gradient}, \ref{eq:PG_grad}), actor$_g$ (\ref{eq:sutton_grad}, \ref{eq:PG_grad}) and actor$_o$ (\ref{eq:PI_grad_cont}) gradients
\begin{align*}
\nabla_\theta J 
&=H_\theta\Delta(\Xi^\top\vd_{\Scal,\theta})\vq_\theta
&&= \sum_{s} d_{\Scal,\theta}(s) \sum_a q_\theta(s, a) \partial_\theta \pi_\theta(s,a)
\\
\nabla_\theta^\phi J
&=H_\theta\Delta(\Xi^\top\vd_{\Scal,\theta})\vq_\phi
&&= \sum_{s} d_{\Scal,\theta}(s) \sum_a q_\phi(s, a) \partial_\theta \pi_\theta(s,a)
\\
\partial_\theta J_\pi 
&=(1-\gamma)H_\theta\Delta(\Xi^\top\vmu_0)\vq_\phi
&&=(1-\gamma)\sum_s \mu_0(s)
\sum_a q_\phi(s,a)\partial_\theta\pi_\theta(s,a)
\end{align*}
\item The gap between policy gradient and actor gradients (\ref{eq:three_term_obj}, \ref{eq:pg_actor_g_diff})
\begin{align*}
\nabla_\theta J 
&=\nabla_\theta^\phi J
+ \partial_\theta (\vd^\top_\theta\vdelta'_{\theta,\phi})
&&=\nabla_\theta^\phi J
+\partial_\theta\EE_{(S,A)\sim d_\theta}
[\delta'_{\theta,\phi}(S,A)]
\\
&=\partial_\theta J_\pi 
+ \partial_\theta (\vd^\top_\theta\vdelta_{\theta,\phi})
&&=\partial_\theta J_\pi 
+\partial_\theta\EE_{(S,A)\sim d_\theta}
[\delta_{\theta,\phi}(S,A)]
\end{align*}
\end{itemize}

\section{Proofs}
\label{app:proofs}

\stationary*

\begin{proof}
By the chain rule, we can calculate $\Upsilon$ as
\begin{align}
\Upsilon=H_\theta\widetilde{\Upsilon}
\label{eq:chain_rule_upsilon}
\end{align}
where $\widetilde{\Upsilon}_{\widetilde{s}\widetilde{a},sa}=\frac{\partial d_\theta(s,a)}{\partial \pi_\theta(\widetilde{s},\widetilde{a})}$ and recall that $(H_\theta)_{i,sa}=\frac{\partial \pi_\theta(s,a)}{\partial\theta_i}$ as defined in \cref{eq:actor_jacobian}. 
Next we show how to calculate $\widetilde{\Upsilon}$ using the implicit function theorem. 

Based on \cref{eq:d_recursion_matrix}, define $\vf(\vd,\vpi)$ as
\begin{align}
\vf(\vd,\vpi) = (I-\gamma\Pi^\top P^\top)\vd-(1-\gamma)\Pi^\top\vmu_0
\end{align}
We know from \cref{eq:d_recursion_matrix} that $\vf(\vd_\theta,\vpi_\theta)=\zerovec_{|\Scal||\Acal|}$.
The Jacobian w.r.t.\ $\vd$ at $(\vd_\theta,\vpi_\theta)$ is
\begin{align}
\left.\frac{\partial\vf}{\partial\vd}\right|_{\vd=\vd_\theta,\vpi=\vpi_\theta} 
= I-\gamma\Pi_\theta^\top P^\top = \Psi_\theta^\top
\end{align}
which is invertible under regular assumptions (recall that $\gamma<1$).
As for $\left.\frac{\partial \vf}{\partial \vpi}\right|_{\vd=\vd_\theta,\vpi=\vpi_\theta}$, 
we can see that it is diagonal because
\begin{align}
f(s,a)
= d(s,a)-\gamma\pi(s,a)
\sum_{\widetilde{s},\widetilde{a}}
P(s|\widetilde{s},\widetilde{a})d(\widetilde{s},\widetilde{a})
-(1-\gamma)\pi(s,a)\mu_0(s)
\end{align}
and it does not depend on policy values other than $\pi(s,a)$.
Then
\begin{align}
\left.\frac{\partial f(s,a)}{\partial\pi(s,a)}
\right|_{\vd=\vd_\theta,\vpi=\vpi_\theta} 
= -\gamma\sum_{\widetilde{s},\widetilde{a}}
P(s|\widetilde{s},\widetilde{a})d_\theta(\widetilde{s},\widetilde{a})
-(1-\gamma)\mu_0(s)
=-d_{\Scal,\theta}(s)
\end{align}
where the last equality is due to \cref{eq:d_recursion_matrix}.
Broadcasting it to all actions, we get 
\begin{align}
\left.\frac{\partial \vf}{\partial\vpi}
\right|_{\vd=\vd_\theta,\vpi=\vpi_\theta} 
= -\Delta(\Xi^\top\vd_{\Scal,\theta})
\end{align}
Then by the implicit function theorem, we have
\begin{align}
\widetilde{\Upsilon}^\top
&=
-\left(\left.\frac{\partial \vf}{\partial\vd}
\right|_{\vd=\vd_\theta,\vpi=\vpi_\theta}\right)^{-1}
\left(\left.\frac{\partial \vf}{\partial\vpi}
\right|_{\vd=\vd_\theta,\vpi=\vpi_\theta}\right)
\\
&= (\Psi_\theta^\top)^{-1}\Delta(\Xi^\top\vd_{\Scal,\theta})
\\
\Longrightarrow \widetilde{\Upsilon}
&= \Delta(\Xi^\top\vd_{\Scal,\theta})\Psi_\theta^{-1}
\end{align}
This combined with \cref{eq:chain_rule_upsilon} completes the proof.
\end{proof}

\stackelberg*

\begin{proof}
The first term can be computed as
\begin{align}
\partial_q J_\pi
&=(1-\gamma)\Pi_\theta^\top\vmu_0.
\label{eq:D2f1}
\end{align}
Plugging \cref{eq:d_recursion_matrix,eq:D2f1,eq:D22f2} to \cref{eq:stack_grad_pi} gives
\begin{align}
\vg_{S,\theta}&=\partial_\theta J_\pi
-(\partial_\theta\partial_q J_q)^\top
(\Psi_\theta^\top D_\theta \Psi_\theta)^{-1}
(I-\gamma\Pi_\theta^\top P^\top)\vd_\theta 
\\
&=\partial_\theta J_\pi
-(\partial_\theta\partial_q J_q)^\top
\Psi_\theta^{-1}D_\theta^{-1} \vd_\theta 
\\
&=\partial_\theta J_\pi 
- \frac{1}{1-\gamma}(\partial_\theta\partial_q J_q)^\top\onevec_{|\Scal||\Acal|}
\end{align}
where the last equation is due to 
\begin{align}
\Psi_\theta^{-1}\onevec_{|\Scal||\Acal|}
=\left(\sum_{i=0}^\infty (\gamma P\Pi_\theta)^i\right)
\onevec_{|\Scal||\Acal|}
=\sum_{i=0}^\infty \gamma^i 
\left((P\Pi_\theta)^i\onevec_{|\Scal||\Acal|}\right)
=\sum_{i=0}^\infty \gamma^i 
\onevec_{|\Scal||\Acal|}
=\frac{1}{1-\gamma}\onevec_{|\Scal||\Acal|}
\end{align}
Under natural regularity assumptions that allow transposing derivative with summation, we can compute the gradient using \cref{eq:D2f2} as
\begin{align}
\vg_{S,\theta}
&=\partial_\theta J_\pi
+\frac{1}{1-\gamma}
\left[\partial_\theta\left(
		\Psi_\theta^\top D_\theta \vdelta_\theta
\right)\right]^\top \onevec_{|\Scal||\Acal|}
\\
&=\partial_\theta J_\pi
+\frac{1}{1-\gamma}
\partial_\theta\left[\left(
		\Psi_\theta^\top D_\theta \vdelta_\theta
\right)^\top \onevec_{|\Scal||\Acal|} \right]
\\
&=\partial_\theta J_\pi
+\partial_\theta\left[  \vdelta_\theta^\top D_\theta 
\onevec_{|\Scal||\Acal|} \right]
\qquad (\text{Note that } \Psi_\theta\onevec_{|\Scal||\Acal|}=(1-\gamma)\onevec_{|\Scal||\Acal|})
\\
&=\partial_\theta J_\pi
+\partial_\theta(\vd^\top_\theta\vdelta_\theta)
\end{align}
Therefore, the Stackelberg gradient 
$\vg_{S,\theta}$ 
is not only doing (partial) policy improvement, but also maximizing $\vd_\theta^\top\vdelta_\theta$.

Adding this term to the actor objective, apply \cref{eq:d_recursion_matrix}, and get the Stackelberg actor objective
\begin{align}
\max_{\vtheta}\ J_{S,\pi}(\vtheta,\vq)
&=(1-\gamma)\vmu_0^\top \Pi_\theta \vq
+\vd^\top_\theta\vdelta_\theta
\label{eq:stack_actor_obj}
\\
&=\vd_\theta^\top\Psi_\theta\vq
+\vd^\top_\theta(\vr-\Psi_\theta\vq)
\\
&=\vd_\theta^\top\vr
\end{align}
which is exactly the dual objective of the cumulative reward objective in \cref{eq:cumulative_obj}. 
This means the Stackelberg gradient is unbiased policy gradient. 
\end{proof}

\section{Soft Actor-Critic}
\label{app:soft_ac}

\subsection{Derivation of Res-SAC}
In this section, we derive the actor update in Res-SAC.

With an additional entropy term, the cumulative reward objective of SAC is given by
\begin{align}
J^{\text{SAC}}(\vtheta)
&=(1-\gamma)\vmu_0^\top\Pi_\theta\sum_{i=0}^\infty
(\gamma P\Pi_\theta)^i(\vr-\log\vpi_\theta)\\
&=(1-\gamma)\vmu_0^\top\Pi_\theta
\left[
\left(\vr+\sum_{i=1}^\infty
(\gamma P\Pi_\theta)^i(\vr-\log\vpi_\theta)
\right)
-\log\vpi_\theta\right]\\
&=(1-\gamma)\vmu_0^\top\Pi_\theta
(\vq_\theta-\log\vpi_\theta)\\
&=\vd_\theta^\top (\vr-\log\vpi_\theta)
\label{eq:ent_reg_obj}
\end{align}
where $\vq_\theta\defeq\vr+\sum_{i=1}^\infty
(\gamma P\Pi_\theta)^i(\vr-\log\vpi_\theta)$ is the action value accounting for all future entropy terms but \emph{excluding} the current entropy term (as defined as in the SAC paper).
Using a critic $\vq_\phi$, the actor and critic objectives are
\begin{alignat}{2}
\text{Actor}:\quad \max_{\vtheta}\ 
J^{\text{SAC}}_\pi(\vtheta,\vphi)
&=(1-\gamma)\vmu_0^\top\Pi_\theta
(\vq_\phi-\log\vpi_\theta)
\label{eq:sac_actor_obj}\\
\text{Critic}:\quad 
\min_{\vphi}\ 
J^{\text{SAC}}_q(\vtheta,\vphi)
&=\frac{1}{2}\|\vr+
\gamma P\Pi_\theta(\vq_\phi-\log\vpi_\theta)-\vq_\phi\|_\vd^2
\label{eq:sac_critic_obj}
\end{alignat}
Note that the original SAC paper uses a KL divergence minimization step for the actor update~\citep[Eq.(10)]{haarnoja2018soft}, which can be seen as a variant of \cref{eq:sac_actor_obj}.
More specifically, one can re-express the KL divergence as a Bregman divergence associated with the negative entropy $-\Hcal$~\citep{bridge}:
\begin{align}
\max_{\vtheta}\ -\EE_{S\sim\vmu_0}
\left[KL\left(\vpi_{\theta}(S,\cdot)
\middle\|\frac{\exp(\vq_{\phi}(S,\cdot))}{Z_\phi(S)}\right)\right]
&=\EE_{S\sim\vmu_0}\left[
-\vpi_\theta(S,\cdot)^\top
\nabla\Hcal\left(\frac{\exp(\vq_{\phi}(S,\cdot))}{Z_\phi(S)}\right)
+\Hcal(\vpi_\theta(S,\cdot))+C\right]
\\
&=\EE_{S\sim\vmu_0}\left[
\vpi_\theta(S,\cdot)^\top\vq_\phi(S,\cdot)
+\Hcal(\vpi_\theta(S,\cdot))+C'\right]
\\
&=
\vmu_0^\top\Pi_\theta
(\vq_\phi-\log\vpi_\theta)+C'
\end{align}
where $Z_\phi(s)$ is the partition function for state $s$ and $C,C'$ are some constants independent of $\vtheta$.
Thus \cref{eq:sac_actor_obj} is the same as the KL divergence minimization up to some constants and rescaling.

The first order derivatives of \cref{eq:sac_actor_obj,eq:sac_critic_obj} are
\begin{align}
\partial_\theta J^{\text{SAC}}_\pi
&=(1-\gamma)\EE_{S\sim\mu_0}
\left[\partial_\theta\sum_A\pi_\theta(S,A)(q_\phi(S,A)
-\log\pi_\theta(S,A))\right]
\\
&=(1-\gamma)\EE_{S\sim\mu_0,\epsilon}
\left[\partial_\theta\log\pi_\theta(S,A)
+\partial_A(q_\phi(S,A)
-\log\pi_\theta(S,A))\partial_\theta\pi_\theta(S,\epsilon)\right]
\label{eq:sac_first_term_grad}\\
\partial_{\phi} J^{\text{SAC}}_\pi
&=(1-\gamma)\Pi_\theta^\top\vmu_0
\\
\partial_{\phi} J^{\text{SAC}}_q
&=-\Psi_\theta^\top D\vdelta_E
\end{align}
where \cref{eq:sac_first_term_grad} uses the reparametrization trick~\citep{schulman2015gradient} with $a=\pi_\theta(s,\epsilon)$ for some random variable $\epsilon$, and $\vdelta_E=\vr+\gamma P\Pi_\theta(\vq_\phi-\log\vpi_\theta)-\vq_\phi$ is the residual of the critic accounting for the entropy of the next state.
The reparametrization trick is needed because the policy is predicting an action in a continuous action space.

The gap between $J^{\text{SAC}}(\vtheta)$ and $J_\pi^{\text{SAC}}(\vtheta,\vphi)$ is
\begin{align}
J^{\text{SAC}}(\vtheta)-J_\pi^{\text{SAC}}(\vtheta,\vphi)
&=\vd_\theta^\top (\vr-\log\vpi_\theta)
-(1-\gamma)\vmu_0^\top\Pi_\theta
(\vq_\phi-\log\vpi_\theta)
\\
&=\vd_\theta^\top (\vr-\log\vpi_\theta)
-\vd_\theta^\top\Psi_\theta
(\vq_\phi-\log\vpi_\theta)
\\
&=\vd^\top_\theta\vdelta_E
\label{eq:sac_obj_gap}
\end{align}
The question now becomes how to compute $\partial_\theta(\vd_\theta^\top\vdelta_E)$.
As in \cref{sec:gap_obj} (\ref{eq:d_delta_product_grad}), by the product rule, we have
\begin{align}
\partial_\theta(\vd_\theta^\top\vdelta_E)
=
\partial_\theta((\vd_\theta')^\top\vdelta_E)
+\partial_\theta(\vd_\theta^\top\vdelta'_E)
\end{align}
Using the reparametrization trick, $\partial_\theta((\vd_\theta')^\top\vdelta_E)$ can be computed as
\begin{align}
\EE_{(S,A)\sim d_\theta}[\partial_\theta\delta_E(S,A)]
&=\EE_{(S,A)\sim d_\theta,\widetilde{S}\sim P}
\left[\gamma\partial_\theta\sum_{\widetilde{A}}
\pi_\theta(\widetilde{S},\widetilde{A})(q_\phi(\widetilde{S},\widetilde{A})
-\pi_\theta(\widetilde{S},\widetilde{A}))\right]
\\
&=\gamma\EE_{(S,A)\sim d_\theta,\widetilde{S}\sim P,\widetilde{\epsilon}}
\left[\partial_\theta\log\pi_\theta(\widetilde{S},\widetilde{A})
+\partial_{\widetilde{A}}(q_\phi(\widetilde{S},\widetilde{A})-\log\pi_\theta(\widetilde{S},\widetilde{A}))
\partial_\theta\pi_\theta(\widetilde{S},\widetilde{\epsilon})\right]
\label{eq:sac_third_term_grad}
\end{align}
It can be combined with \cref{eq:sac_first_term_grad}, using the \cref{eq:d_recursion_matrix}, to get
\begin{align}
\partial_\theta[(1-\gamma)\vmu_0^\top\Pi_\theta(\vq_\phi-\log\vpi_\theta)
+(\vd'_\theta)^\top\vdelta_E]
=\EE_{(S,A)\sim d_\theta}
\left[\partial_\theta\log\pi_\theta(S,A)
+\partial_{A}(q_\phi(S,A)-\log\pi_\theta(S,A))
\partial_\theta\pi_\theta(S,\epsilon)\right]
\label{eq:sac_first_and_third_term_grad}
\end{align}
\cref{eq:sac_first_and_third_term_grad} explains the original SAC implementation for the actor update, except for using $\vd_\theta$ instead of a replay buffer to compute the expectation.
The final term, $\partial_\theta(\vd_\theta^\top\vdelta'_E)$, is optimizing a policy to maximize $\vdelta_E$ as a fixed reward. 
It is also the residual reward for learning a res-critic for Res-SAC.

\subsection{Derivation of Stack-SAC}
In this section, we derive the actor update in Stack-SAC.

The second order derivative $\partial_\phi^2 J_q$ remains the same as \cref{eq:D22f2}.
The Stackelberg gradient now becomes
\begin{align}
\vg^{\text{SAC}}_{S,\theta}
=\partial_\theta J_\pi
+\partial_\theta
(\Psi_\theta^\top D\vdelta_E)^\top\Psi_\theta^{-1}D\vd_\theta
=\partial_\theta J_\pi+\partial_\theta (\vd_\theta^\top\vdelta_E).
\end{align}
Then based on \cref{eq:sac_obj_gap}, $\vg^{\text{SAC}}_{S,\theta}=\nabla_\theta J^{\text{SAC}}$ is an unbiased policy gradient for SAC.

\section{Experiment Details}
\label{app:exp_details}

\newcommand{\subfigheight}{0.26\textwidth}

\subsection{Tabular}
\cref{subfig:fourroom_env} shows the FourRoom environment where goal is to reach a particular cell. 
The initial state distribution is a uniform distribution over all unoccupied cells.
The reward is 1 for reaching the goal and 0 otherwise.

\subsubsection{Additional Dynamic Programming Experiments}

We investigate the effects of different gradient methods (Actor$_o$, Actor$_g$ or Policy Gradient (PG)) combined with different critic updates (Bellman residual minimization (BR) or temporal difference iteration (TD)) in the dynamic programming setting, where the reward $\vr$ and the transition $P$ are assumed to be known.
This showcases the performance of different algorithms in the ideal scenario. 

Specifically, the parameters of the softmax policy are the logits, and the critic is directly parametrized (i.e., it is tabular with a scaler value for each state-action pair). 
Using $\vr$ and $P$, one can compute the Actor$_o$ gradient $\partial_\theta J_\pi$ (\ref{eq:PI_grad_cont}), Actor$_g$ gradient $\nabla_\theta^\phi J$ (\ref{eq:sutton_grad}) and PG (\ref{eq:policy_gradient}) directly, and apply them to the policy parameters.
As for the critic, BR uses the full gradient (\ref{eq:D2f2}) while TD uses semi-gradient (\ref{eq:semi_critic_grad}).
The critic TD error loss is weighted according to the on-policy distribution $\vd_\theta$.
Both the actor and the critic use Adam optimizer, with respective learning rates of 0.01 and 0.02.

\cref{subfig:fr_Jpi_DP,subfig:fr_Jq_DP} show the average return of the policy $J_\pi$ and the critic objective value $J_q$ respectively.
There are a few observations.
(1)~Policy gradient quickly converges to the optimal performance, regardless of whether BR or TD is used.
The performance of PG+BR and PG+TD are hardly distinguishable in \cref{subfig:fr_Jpi_DP}, showing that if one can estimate PG accurately, the critic update may be less important.
(2)~Actor-Critic~(AC) is slower than PG to achieve optimal performance.
Even with much more iterations, both Actor$_o$-Critic (A$_o$C) and Actor$_o$-Critic~(A$_g$C) are very slow to reach the final performance of PG.
Furthermore, this happens even when the critic is providing accurate estimates of the $Q$-values (as $J_q$ is very small after 2000 iterations).
This indicates that PG can be a better choice as long as one has access to it, and our effort of closing the gap between AC and PG is meaningful in practice.
(3)~TD performs better than BR for both A$_o$C and A$_g$C. 
However, this is not always the case~\citep{scherrer2010should}.

\begin{figure*}[t]
\centering
\subcaptionbox{The FourRoom environment
\label{subfig:fourroom_env}}{%
  \includegraphics[height=\subfigheight]{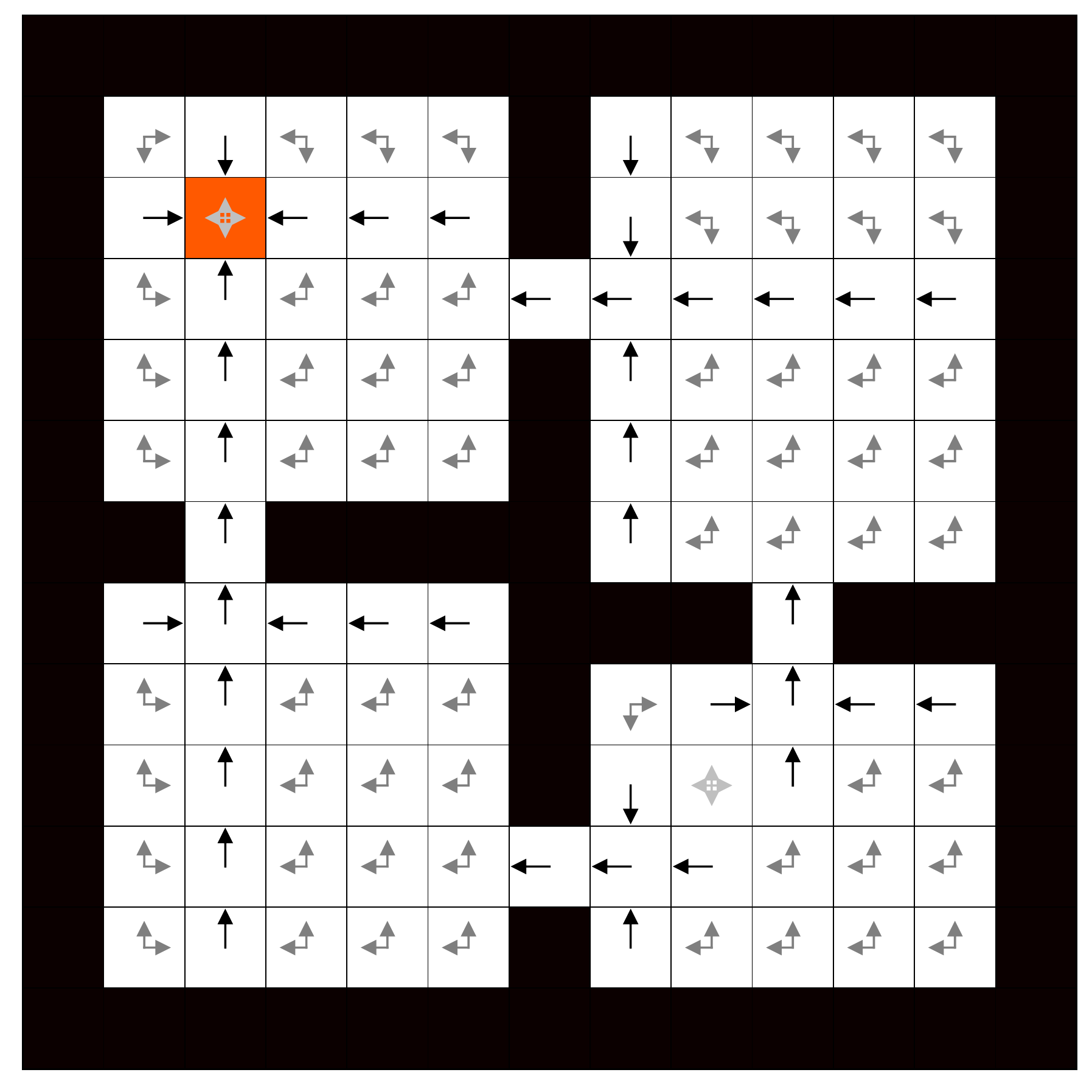}
}
\subcaptionbox{Policy performance $J_\pi$
\label{subfig:fr_Jpi_DP}}{%
  \includegraphics[height=\subfigheight]{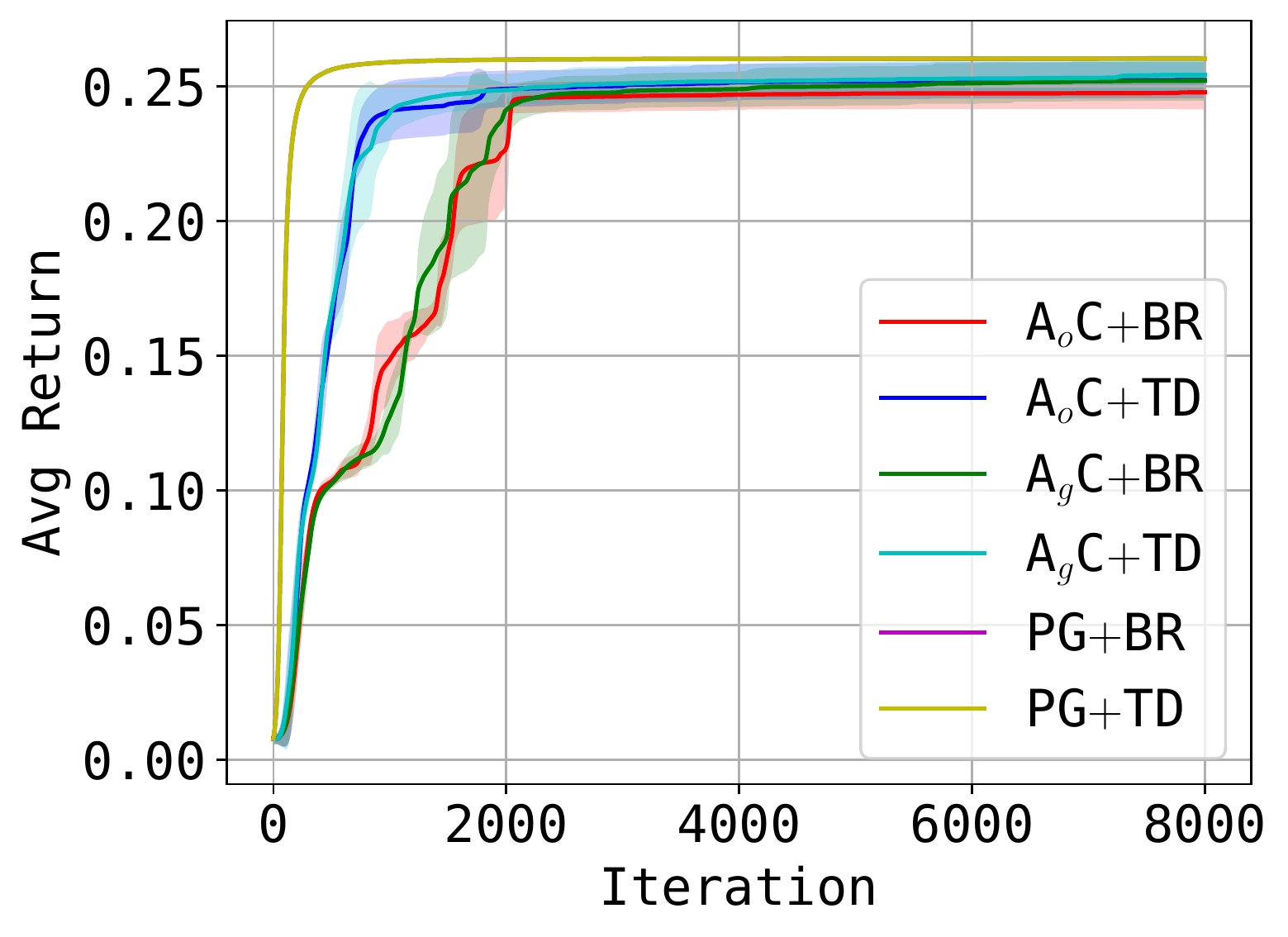}
}
\subcaptionbox{Objective value $J_q$
\label{subfig:fr_Jq_DP}}{%
  \includegraphics[height=\subfigheight]{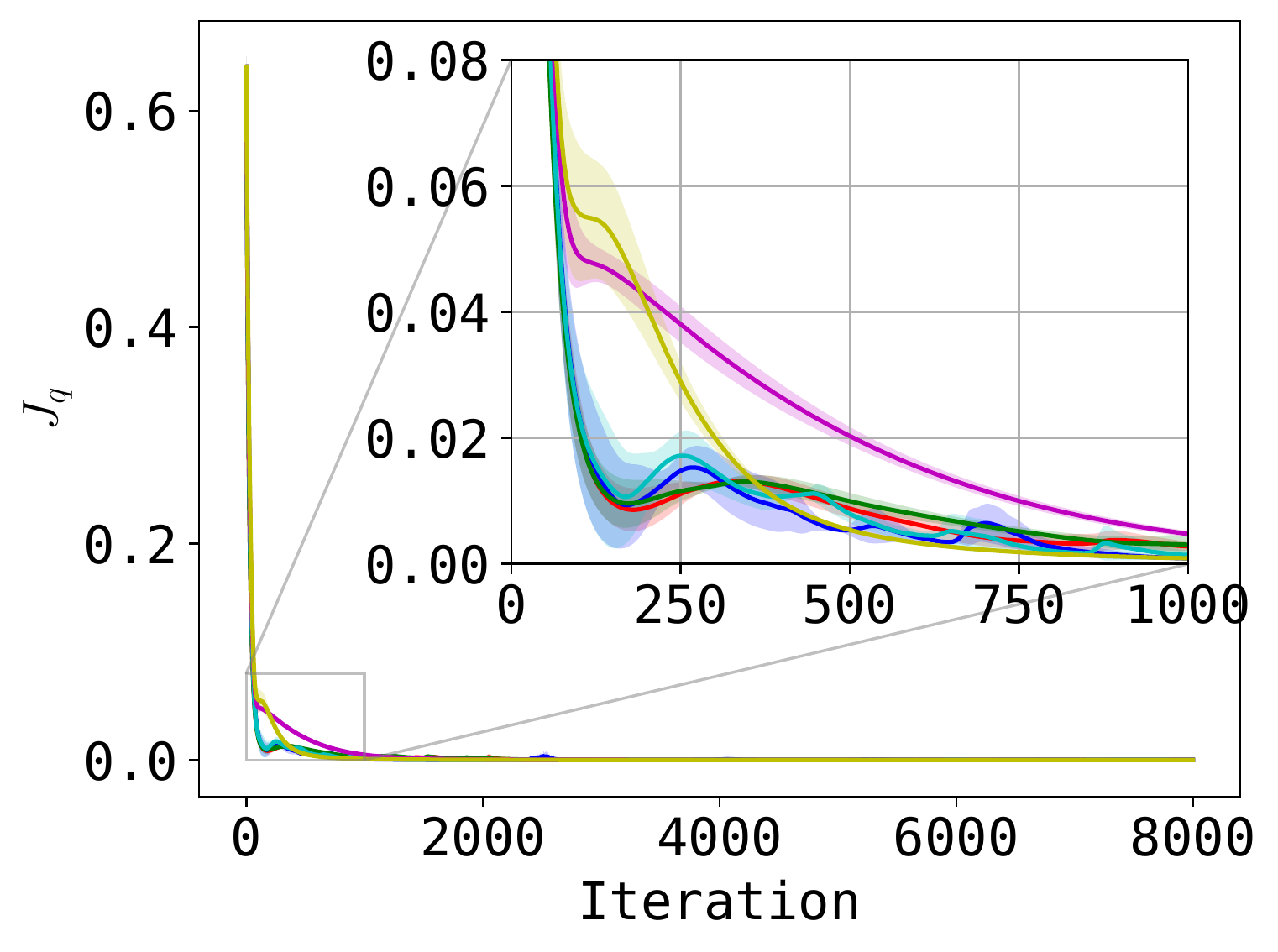}
}
\caption{FourRoom with a goal and results of dynamic programming}
\label{fig:four_room_env}
\end{figure*}

\subsubsection{Sample-Based Experiments}
\begin{table}
\centering
\begin{tabular}{ |p{6cm}||p{6cm}|  }
 \hline
 \multicolumn{2}{|c|}{Hyperparameters for four-room domain experiments} \\
 \hline
 Parameter & Value\\
 \hline
 optimizer   & Adam\\
 learning rate for actor/policy &   $1 \cdot 10^{-2}$ \\
 learning rate for critic &   $2 \cdot 10^{-2}$ \\
 discount ($\gamma$) & $0.9$ \\
 environment steps per gradient update & $300$ (episode length) \\
 batch size & $300$ \\
 Stack-AC: value of $\eta$ & $0.5$ \\
 Res-AC: learning rate for res-critic &   $2 \cdot 10^{-2}$ \\
 \hline
\end{tabular}
\caption{Hyperparameters used for Actor$_o$-Critic, Actor$_g$-Critic, Stack-AC, and Res-AC for the sample-based experiments on the four-room domain.}
\label{tab:sharedparamsfourroom}
\end{table}

\begin{figure*}
\begin{minipage}{0.45\textwidth}
\begin{algorithm}[H]
\small
\centering
\caption{Actor$_o$-Critic}
\begin{algorithmic}
\State Initialize $\vtheta$, $\vphi$
\State $\mathcal{O} \leftarrow \varnothing$
\For{each episode}
    \State $\mathcal{D} \leftarrow \varnothing$
    \State $s_0 \sim \mu_0$
    \State $\mathcal{O} \leftarrow \mathcal{O} \cup \{s_0\}$
    \For{each environment step}
        \State $a_t \sim \pi(a_t | s_t)$ 
        \State $s_{t+1} \sim P(s_{t+1} | s_t, a_t)$
        \State $\mathcal{D} \leftarrow \mathcal{D} \cup \{(s_t, a_t, r(s_t, a_t), s_{t+1}, a_{t+1})\}$
    \EndFor
    \State $\vtheta \leftarrow \vtheta + \alpha_\theta \partial_\theta \widetilde{J}_\pi^o$ 
    \State $\vphi \leftarrow \vphi - \alpha_\phi \partial_\phi \widetilde{J}_q$
\EndFor
\end{algorithmic}
\label{alg:actor_o_critic}
\end{algorithm}
\end{minipage}
\hfill
\begin{minipage}{0.45\textwidth}
\begin{algorithm}[H]
\centering
\small
\caption{Actor$_g$-Critic}
\begin{algorithmic}
\State Initialize $\vtheta$, $\vphi$
\For{each episode}
    \State $\mathcal{D} \leftarrow \varnothing$
    \For{each environment step}
        \State $a_t \sim \pi(a_t | s_t)$ 
        \State $s_{t+1} \sim P(s_{t+1} | s_t, a_t)$
        \State $\mathcal{D} \leftarrow \mathcal{D} \cup \{(s_t, a_t, r(s_t, a_t), s_{t+1}, a_{t+1})\}$
    \EndFor
    \State $\vtheta \leftarrow \vtheta + \alpha_\theta \partial_\theta \widetilde{J}_\pi^g$ 
    \State $\vphi \leftarrow \vphi - \alpha_\phi \partial_\phi \widetilde{J}_q$
\EndFor
\end{algorithmic}
\label{alg:actor_g_critic}
\end{algorithm}
\end{minipage}
\vspace{-15pt}
\end{figure*}

\begin{figure*}
\begin{minipage}{0.45\textwidth}
\begin{algorithm}[H]
\small
\centering
\caption{Stack-AC}
\begin{algorithmic}
\State Initialize $\vtheta$, $\vphi$
\State $\mathcal{O} \leftarrow \varnothing$
\For{each episode}
    \State $\mathcal{D} \leftarrow \varnothing$
    \State $s_0 \sim \mu_0$
    \State $\mathcal{O} \leftarrow \mathcal{O} \cup \{s_0\}$
    \For{each environment step}
        \State $a_t \sim \pi(a_t | s_t)$ 
        \State $s_{t+1} \sim P(s_{t+1} | s_t, a_t)$
        \State $\mathcal{D} \leftarrow \mathcal{D} \cup \{(s_t, a_t, r(s_t, a_t), s_{t+1}, a_{t+1})\}$
    \EndFor
    \State $\vtheta \leftarrow \vtheta + \alpha_\theta \widetilde{\vg}_{S,\theta}^{\text{semi}}$ 
    \State $\vphi \leftarrow \vphi - \alpha_\phi \partial_\phi \widetilde{J}_q$
\EndFor
\end{algorithmic}
\label{alg:stack_ac}
\end{algorithm}
\end{minipage}
\hfill
\begin{minipage}{0.45\textwidth}
\begin{algorithm}[H]
\centering
\small
\caption{Res-AC}
\begin{algorithmic}
\State Initialize $\vtheta$, $\vphi$
\For{each episode}
    \State $\mathcal{D} \leftarrow \varnothing$
    \For{each environment step}
        \State $a_t \sim \pi(a_t | s_t)$
        \State $s_{t+1} \sim P(s_{t+1} | s_t, a_t)$
        \State $\mathcal{D} \leftarrow \mathcal{D} \cup \{(s_t, a_t, r(s_t, a_t), s_{t+1}, a_{t+1})\}$
    \EndFor
    \State $\vtheta \leftarrow \vtheta + \alpha_\theta \partial_\theta \widetilde{J}_\pi^\text{res}$ 
    \State $\vphi \leftarrow \vphi - \alpha_\phi \partial_\phi \widetilde{J}_q$
    \State $\vpsi \leftarrow \vpsi - \alpha_\psi \partial_\psi \widetilde{J}_w$
\EndFor
\end{algorithmic}
\label{alg:res_ac}
\end{algorithm}
\end{minipage}
\vspace{-15pt}
\end{figure*}

This section refers to the FourRoom experiments in \cref{sec:fourroom} in the main text. We implement sample-based algorithms which follow the Actor$_o$-Critic~(A$_o$C), Actor$_g$-Critic~(A$_g$C), Stack-AC, and Res-AC updates. Hyper-parameters are listed in \cref{tab:sharedparamsfourroom}.

A$_o$C uses the following procedure. We use two replay buffers, a replay buffer $\mathcal{O}$ which will store states from the initial state distribution of the environment, and a replay buffer $\mathcal{D}$ which will store transitions from the most recent episode run using the current policy. Concretely, $\mathcal{O}$ stores samples from $\vmu_0$ whereas $\mathcal{D}$ stores samples from the on-policy distribution $\vd_\theta$. At the beginning of the training procedure, we initialize $\mathcal{O}$ to be empty. At the beginning of each episode, we initialize $\mathcal{D}$ to be empty, and we add the initial state sampled from the environment to the initial state replay buffer $\mathcal{O}$. For each step in the environment before the episode terminates, we add the current transition to the replay buffer $\mathcal{D}$. After the episode terminates, we apply gradient updates to the actor (policy) and the critic. 
Stack-AC follows this same procedure but with different actor updates. 
A$_g$C and Res-AC follow the same procedure except for not having an initial state buffer and using different actor updates.
Res-AC additionally includes a res-critic update. 
For all algorithms, we compute a single critic update (and res-critic update for Res-AC) for each actor-update to product \cref{fig:four_room_sample}. 
To produce \cref{fig:four_room_res_effect}, we updated the res-critic $5$ times for each critic update to obtain a more accurate res-critic in order to illustrate the res-critic's ability to close the gap between the critic's prediction and the true return. 

For all algorithms we compute the gradient update for the critic as follows. We sample a batch of transitions $\mathcal{B}_D$ from the replay buffer $\mathcal{D}$ and compute the following loss function for the critic to minimize:
\begin{align*}
    \widetilde{J}_q = \frac{1}{|\mathcal{B}_D|}\sum_{(s, a, r, \widetilde{s}, \widetilde{a}) \in \mathcal{B}_D} 
    (q_\phi(s, a) - (r + \gamma q'_\phi(\widetilde{s}, \widetilde{a})))^2 
\end{align*}
where $\widetilde{s}, \widetilde{a}$ are the next state and next action in the transition, and $q'_\phi(\widetilde{s}, \widetilde{a})$ indicates that no gradients pass through $q_\phi(\widetilde{s}, \widetilde{a})$, i.e. it is treated as a target network. The gradient of this loss, $\partial_\phi \widetilde{J}_q$, is a sample-based estimate of $\partial_\phi J_q$ (\ref{eq:D2f2}). 
Note, however, that it uses a semi-gradient instead of Bellman residual minimization in $J_q$ (which yields the full/total gradient).

\textbf{Actor$_o$-Critic}

To compute the gradient update for the actor in A$_o$C, we sample a batch of initial states $\mathcal{B}_O$ from the replay buffer $\mathcal{O}$ and compute the following objective function for the actor to maximize: 
\begin{align*}
    \widetilde{J}_\pi^o = \frac{1}{|\mathcal{B}_O|}\sum_{s \in \mathcal{B}_O} \log \pi_\theta(a | s) q_\phi(s, a).
\end{align*}
where the action $a$ is sampled from the current policy $\pi_\theta$ for each state $s$ in $\mathcal{B}_O$. The gradient of this objective, $\partial_\theta \widetilde{J}_\pi^o$, is a sample-based estimate of $\partial_\theta J_\pi$ (\ref{eq:PI_grad_cont}). For the complete pseudocode of A$_o$C, see \cref{alg:actor_o_critic}.

\textbf{Actor$_g$-Critic}

To compute the gradient update for the actor in A$_g$C, we sample a batch of transitions $\mathcal{B}_D$ from the replay buffer $\mathcal{D}$ and compute the following objective for the actor to maximize: 
\begin{align*}
    \widetilde{J}_\pi^g = \frac{1}{|\mathcal{B}_D|}\sum_{(s, a, r, \widetilde{s}, \widetilde{a}) \in \mathcal{B}_D} \log \pi_\theta(a | s) q_\phi(s, a).
\end{align*}
The gradient of this objective, $\partial_\theta \widetilde{J}_\pi^g$, is a sample-based estimate of $\nabla_\theta^\phi J$ (\ref{eq:sutton_grad}). 
For the complete pseudocode of Actor$_g$-Critic, see \cref{alg:actor_g_critic}.

\textbf{Stack-AC}

We compute the gradient update for the actor as follows. Then the Stackelberg gradient based on semi-critic-gradient is given by
\begin{align*}
\vg_{S,\theta}^{\text{semi}}
&\defeq\partial_\theta J_\pi
-(\partial_\theta\partial_q^{\text{semi}} J_q)^\top\!
((\partial_q^{\text{semi}})^2 J_q)^{-1}\!
(\partial_q^{\text{semi}} J_\pi)
\end{align*}
In the above equation, we replace $J_q$ with $\widetilde{J}_q$ and $J_\pi$ with $\widetilde{J}_\pi$. This gives us our actor update for Stack-AC:
\begin{align*}
\widetilde{\vg}_{S,\theta}^{\text{semi}}
&\defeq\partial_\theta \widetilde{J}_\pi
-(\partial_\theta\partial_q^{\text{semi}} \widetilde{J}_q)^\top\!
((\partial_q^{\text{semi}})^2 \widetilde{J}_q)^{-1}\!
(\partial_q^{\text{semi}} \widetilde{J}_\pi)
\end{align*}
For the complete pseudocode of Stack-AC, see \cref{alg:stack_ac}.

\textbf{Res-AC}

We compute the gradient update for the res-critic as follows. We sample a batch of transitions $\mathcal{B}_D$ from the replay buffer $\mathcal{D}$ and compute the following loss function for the res-critic to minimize:
\begin{align*}
    \widetilde{J}_w = \frac{1}{|\mathcal{B}_D|}\sum_{(s, a, r, \widetilde{s}, \widetilde{a}) \in \mathcal{B}_D} 
    (w_\psi(s, a) - (\delta'_\theta(s,a) + \gamma w'_\psi(\widetilde{s}, \widetilde{a})))^2 
\end{align*}
where $\delta'_\theta(s,a) = r(s,a) + \gamma q'_\phi(\widetilde{s}, \widetilde{a}) - q'_\phi(s, a)$ is the TD-error computed using the current critic. 
Note that $w'_\psi(\widetilde{s}, \widetilde{a})$ indicates that that no gradients pass through $w_\psi(\widetilde{s}, \widetilde{a})$, i.e.\ it is treated as a target network. 
The gradient of this loss, $\partial_\psi \widetilde{J}_w$, is a sample-based estimate of the gradient of $J_w$ (\ref{eq:res_q_obj}). Note, however, that it uses a semi-gradient instead of Bellman residual minimization in $J_w$ (which yields the full/total gradient). 

To compute the gradient update for the actor in Res-AC, we sample a batch of transitions $\mathcal{B}_D$ from the replay buffer $\mathcal{D}$ and compute the following objective for the actor to maximize: 
\begin{align*}
    \widetilde{J}_\pi^\text{res} = \frac{1}{|\mathcal{B}_D|}\sum_{(s, a, r, \widetilde{s}, \widetilde{a}) \in \mathcal{B}_D} \log \pi_\theta(a | s) q_\phi(s, a) + \log \pi_\theta(a | s) w_\psi(s, a).
\end{align*}
The gradient of this objective, $\partial_\theta \widetilde{J}_{\pi}^{\text{res}}$, is a sample-based estimate of the actor update in Res-AC (\ref{eq:res_ac_a_update}). For the complete pseudocode of Res-AC, see \cref{alg:res_ac}.

\subsection{Continuous Control}
\begin{algorithm}[H]
\small
\centering
\caption{Res-SAC}
\begin{algorithmic}[1]
\State Initialize parameters $\theta$, $\phi$, $\psi$
\State Initialize replay buffer $\Dcal\leftarrow\emptyset$
\For{each iteration}
    \For{each environment step}
        \State $a_t \sim \pi_\theta(a_t | s_t)$ 
        \State $s_{t+1} \sim P(s_{t+1} | s_t, a_t)$
        \State $\Dcal \leftarrow \Dcal \cup \{(s_t, a_t, r(s_t, a_t), s_{t+1}, a_{t+1})\}$
    \EndFor
    \For{each critic step}
        \State $\phi \leftarrow \phi - \alpha_\phi \widehat{\nabla}J_Q(\phi)$
    \EndFor
    \For{each res-critic step}
        \State $\psi \leftarrow \psi - \alpha_\psi \widehat{\nabla}J_W(\psi)$
    \EndFor
    \For{each actor step}
        \State $\theta \leftarrow \theta - \alpha_\theta \widehat{\nabla}J_\pi(\theta)$
    \EndFor
\EndFor
\end{algorithmic}
\label{alg:res_sac}
\end{algorithm}

\begin{table}
\centering
\begin{tabular}{ |p{5cm}||p{5cm}|  }
 \hline
 \multicolumn{2}{|c|}{Hyperparameters for continuous control experiments} \\
 \hline
 Parameter & Value\\
 \hline
 optimizer   & Adam\\
 learning rate &   $3 \cdot 10^{-4}$ \\
 discount ($\gamma$) & $0.99$ \\
 replay buffer size & $10^6$ \\
 number of hidden layers & $2$ \\
 number of hidden units per layer & $128$ \\
 number of samples per minibatch & $128$ \\
 nonlinearity & ReLU \\
 target smoothing coefficient $(\tau)$ & $0.005$ \\
 target update interval & $1$ \\
 entropy target & $- \text{dim} (\mathcal{A})$ \\
 environment steps per gradient step & $10$ \\
 Stack-SAC: value of $\eta$ & $0.5$ \\
 Res-SAC: value of $c$ & $6.0$ for HalfCheetah-v2, $1.0$ for Reacher-v2, $4.0$ for Pendulum-v0\\
 \hline
\end{tabular}
\caption{Hyperparameters used for SAC, Stack-SAC, and Res-SAC for continuous control experiments.}
\label{tab:sharedparamscont}
\end{table}

Our training protocol for SAC, Stack-SAC, and Res-SAC follows the same training protoocl of SAC (\cite{haarnoja2018soft2}). Hyperparameters used for all algorithms are listed in \cref{tab:sharedparamscont}. 

\subsubsection{Res-SAC: Loss Functions}

Below, we present the loss functions and updates for the actor, critic, and res-critic of Res-SAC. 

Similar to SAC, Res-SAC uses a 
parametrized soft Q-function (critic) $Q_\phi(s, a)$, and a tractable policy (actor) $\pi_\theta(a | s)$. 
Additionally, Res-SAC uses a parametrized 
residual Q-function (res-critic) $W_\psi(s, a)$. 
The parameters of these networks are $\phi, \theta$, and $\psi$.

The soft Q-function parameters are trained exactly as in SAC \cite{haarnoja2018soft2}, but we write the objectives again here for clarity.
The soft Q-function (critic) parameters are trained to minimize the soft Bellman residual:
\begin{align*}
    J_Q(\phi) = \mathbb{E}_{(S, A) \sim \mathcal{D}} 
    \left[\frac{1}{2} \left(
    Q_\phi(S, A) - 
        \left(
        r(S, A) + \gamma \mathbb{E}_{S' \sim P, A'\sim\pi_\theta(S')} 
        \left[Q_{\bar{\phi}}(S',A')-\log\pi_\theta(A'|S')\right]
        \right)
    \right)^2\right]
\end{align*}
where $\mathcal{D}$ is a replay buffer containing previously sampled states and actions,
and the target soft $Q$-function $Q_{\bar{\phi}}$ uses an exponential moving average $\bar{\phi}$ of $\phi$ as done in the original SAC.

The residual Q-function have a similar objective, but the key differences are 
(1) the reward is based on the TD error of the critic and 
(2) the there is no entropy term.
Specifically, the residual Q-function (res-critic) parameters are trained to minimize the Bellman residual:
\begin{align*}
    J_W(\psi) = \mathbb{E}_{(S, A) \sim \mathcal{D}} 
    \left[\frac{1}{2} \left(
    W_\psi(S, A) - 
        \left(
        \widetilde{r}(S, A) + \gamma \mathbb{E}_{S' \sim P, A'\sim\pi_\theta(S')} 
        \left[W_{\bar{\psi}}(S', A')\right]
        \right)
    \right)^2\right]
\end{align*}
where $\widebar{\psi}$ is an exponential moving average of $\psi$.
The clipped reward $\widetilde{r}(s, a)$ is computed as follows:
\begin{align*}
    \widetilde{r}(s, a) = \text{clip}(\delta(s, a), -c, c)
    =\min(\max(\delta(s, a),-c), c)\qquad \text{for } c>0
\end{align*}
where 
\begin{align*}
    \delta(s, a) = r(s, a) + \gamma \mathbb{E}_{S' \sim P, A'\sim\pi_\theta(S')} [Q_{\bar{\phi}}(S', A')] - Q_\phi(s, a)
\end{align*}

The actor / policy is trained by minimizing the KL divergence:
\begin{align*}
    J_\pi(\theta) = \mathbb{E}_{S \sim \mathcal{D}}
    \left[{\text{KL}}
    \left(\pi_\theta(\cdot | S) 
    \middle\|
    \frac{\text{exp}\left[Q_\phi(S, \cdot)+Q_\psi(S, \cdot)\right]}{Z_{\phi,\psi}(S)}
\right)\right]
\end{align*}
    
All objectives above can be optimized with stochastic gradients: 
$\widehat{\nabla}J_Q(\phi)$, 
$\widehat{\nabla}J_W(\psi)$, and 
$\widehat{\nabla}J_\pi(\theta)$. 
The pseudocode for Res-SAC can be found in \cref{alg:res_sac}.

\subsubsection{Res-SAC: Sensitivity Analysis}
\label{app:clip_analysis}

The results in \cref{sec:contexp} suggest that actor-critic algorithms enjoy improved sample efficiency and final performance when their actor update rules are modified to follow the {\shortname} framework. 
In the continuous control tasks, we found that the performance of Res-SAC was dependent on the setting of an additional hyper-parameter: the clip value $c>0$ applied to the TD error $\vdelta_\theta$ that is used as the reward for the res-critic. 
On HalfCheetah-v2, we examine the sensitivity of Res-SAC to the clip value on the res-critic's reward (\cref{fig:ablation}).
Without clipping, training is highly unstable. 
Higher clip values improve stability, with a clip value of $6.0$ leading to the best performance.
We found a useful heuristic to select a clip value for Res-SAC: train SAC on the same task and use the maximum absolute TD error of the critic that occurs during training as the clip value $c$ for Res-SAC. 
As an example, we see that when training SAC on HalfCheetah-v2, the max TD error is between $5.0$ and $6.0$ (\cref{fig:sactderror}), and we find that a clip value of $6.0$ leads to the best performance of Res-SAC on HalfCheetah.

\begin{figure}[t]
\centering
\begin{minipage}{\halffigwidth}
  \centering
  \includegraphics[height=0.6\textwidth]{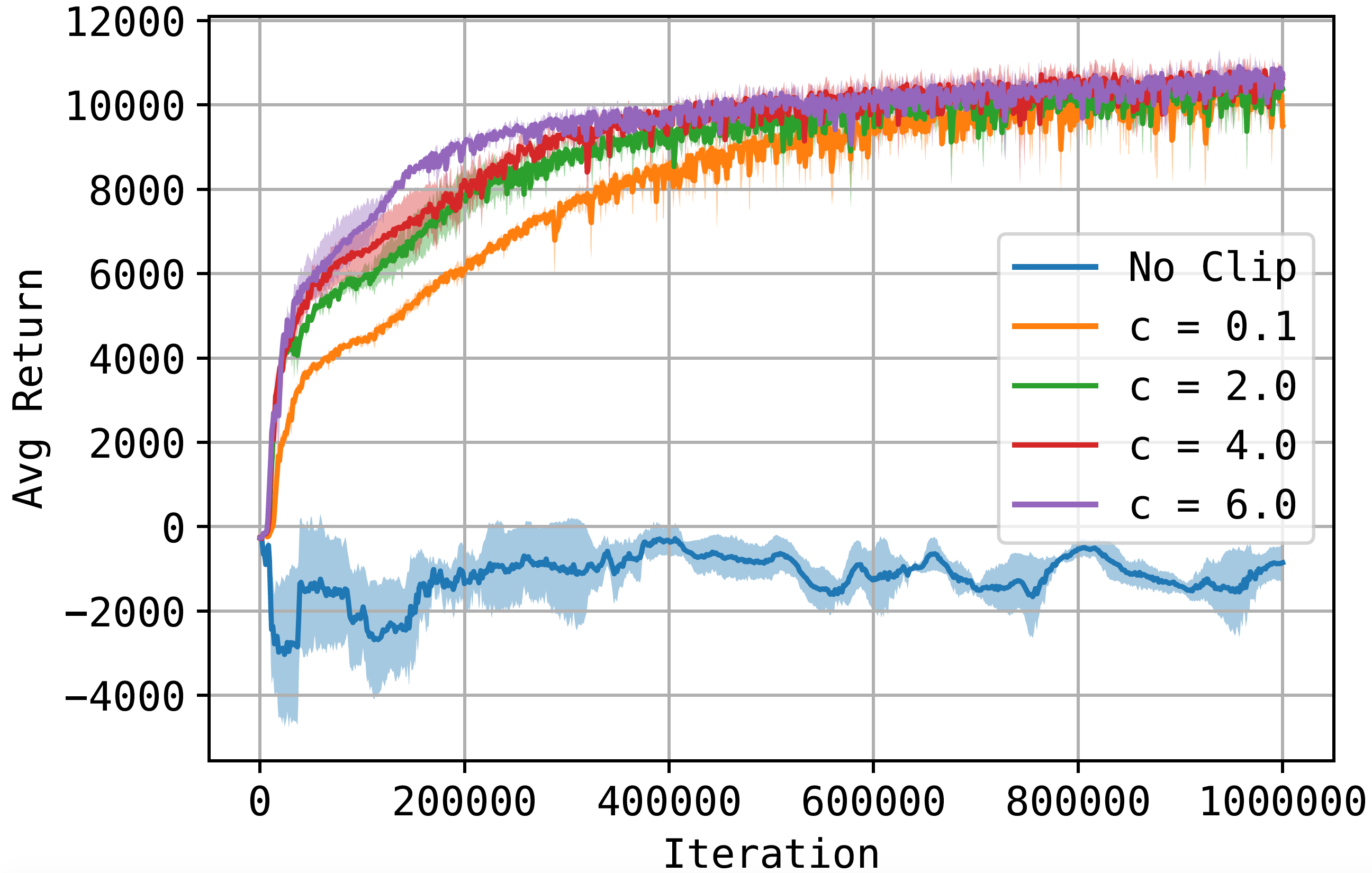}
  \caption{Sensitivity of Res-SAC to the clip value $c$ on the HalfCheetah-v2 task.}
  \label{fig:ablation}
\end{minipage}
\quad %
\begin{minipage}{\halffigwidth}
  \centering
  \includegraphics[height=0.6\textwidth]{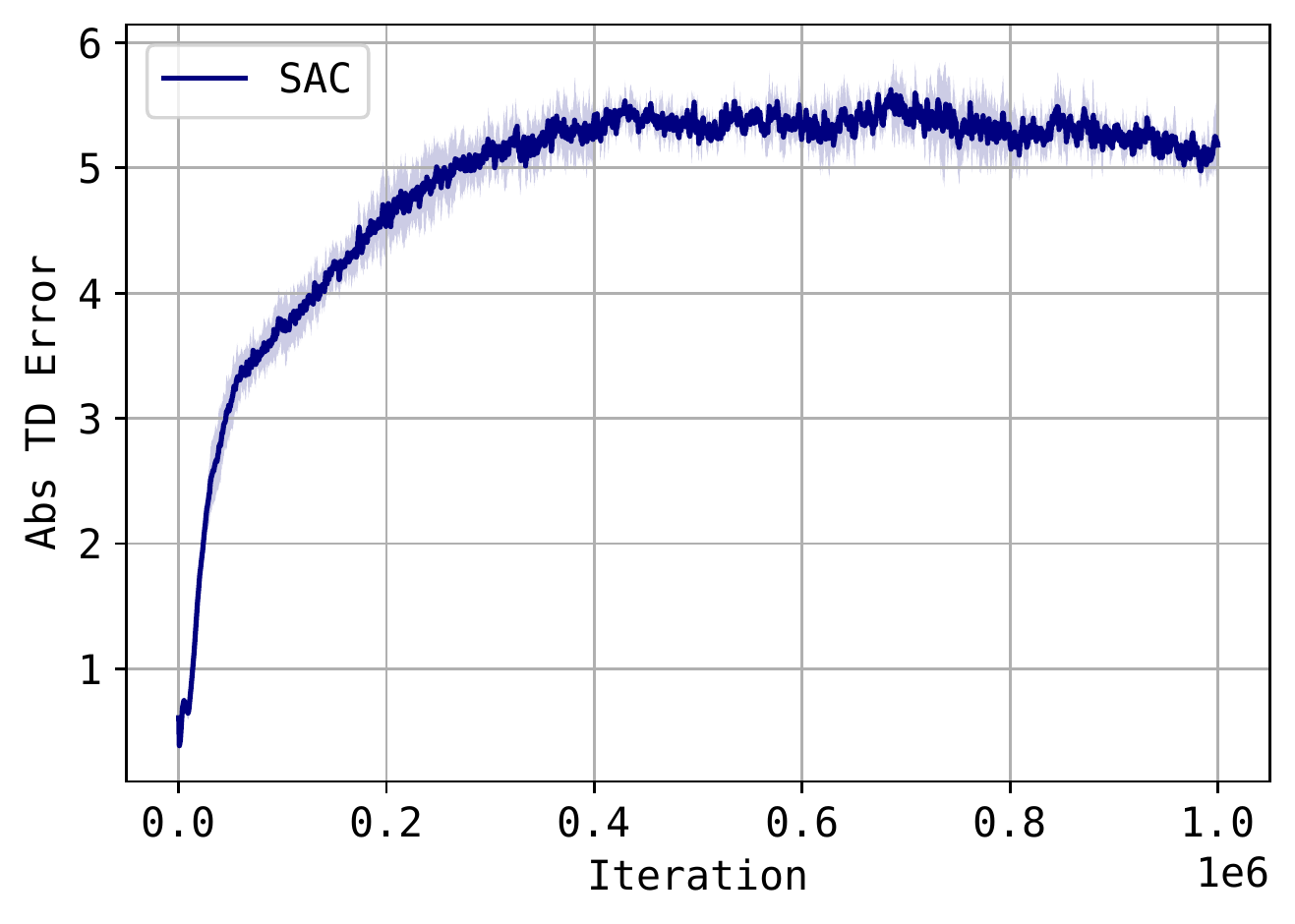}
  \caption{Absolute TD error while training SAC on HalfCheetah-v2.}
  \label{fig:sactderror}
\end{minipage}
\end{figure}

\end{document}